\documentclass[twoside,11pt]{article}

\usepackage{jmlr2e}
\usepackage{footnote}
\usepackage[lined,boxed,commentsnumbered]{algorithm2e}
\usepackage[english]{babel}
\usepackage[noend]{algpseudocode}
\usepackage{amsmath}
\usepackage{mathtools}
\usepackage[utf8]{inputenc}
\usepackage{enumitem}
\usepackage{subfig}
\usepackage{longtable}
\usepackage{tikz}
\allowdisplaybreaks

\newtheorem{thm}{Theorem}
\newtheorem{lem}{Lemma}
\newtheorem{coro}{Corollary}
\newtheorem{defn}{Definition}

\DeclarePairedDelimiter\floor{\lfloor}{\rfloor}
\DeclareMathOperator*{\argmax}{arg\,max}

\newcommand{\norm}[1]{\left\lVert#1\right\rVert}

\newcommand{\RNum}[1]{\uppercase\expandafter{\romannumeral #1\relax}}

\ShortHeadings{Subspace Clustering through Sub-Clusters}{Li, Hannig and Mukherjee}
\firstpageno{1}

\begin{document}
	\title{Subspace Clustering through Sub-Clusters}
	
	\author{\name Weiwei Li \email WEIWEILI@LIVE.UNC.EDU \\
		\addr Department of Statistics and Operations Research\\
		University of North Carolina at Chapel Hill\\
		Chapel Hill, NC 27514, USA \\
		\\       
		\name Jan Hannig \email JAN.HANNIG@UNC.EDU \\
		\addr Department of Statistics and Operations Research\\
		University of North Carolina at Chapel Hill\\
		Chapel Hill, NC 27514, USA \\
		\\
		\name Sayan Mukherjee \email SAYAN@STAT.DUKE.EDU \\
		\addr Department of Statistical Science \\
		Mathematics, Computer Science, Biostatistics \& Bioinformatics \\
		Duke University\\
		Durham, NC 27708, USA 
	}
	\maketitle
	
	\begin{abstract}
		The problem of dimension reduction is of increasing importance in modern data analysis. In this paper, we consider modeling the collection of points in a high dimensional space as a union of low dimensional subspaces. In particular we propose a  highly scalable sampling based algorithm that clusters the entire data via first spectral clustering of a small random sample followed by classifying or labeling the remaining out-of-sample points. The key idea is that this random subset borrows information across the entire dataset and that the problem of clustering points can be replaced with the more efficient problem of ``clustering sub-clusters". We provide theoretical guarantees for our procedure. The numerical results 
		indicate that for large datasets the proposed algorithm outperforms other state-of-the-art subspace clustering algorithms with respect to accuracy and speed.
	\end{abstract}


	\newpage
	
	\section{Introduction}
	In data analysis, researchers are often given datasets with large volume and high dimensionality. To reduce the computational complexity arising in these settings, researchers resort to dimension reduction techniques. To this end, traditional methods like PCA \cite{hotelling1933analysis} use few principal components to represent the original dataset; factor analysis \citep{cattell1952factor} seeks to get linear combinations of latent factors; subsequent works of PCA include kernel PCA \citep{scholkopf1998nonlinear}, generalized PCA \citep{vidal2005generalized}; manifold learning \citep{belkin2003laplacian} assumes data points collected from a high dimensional ambient space lie around a low dimensional manifold, and muli-manifold learning \citep{liu2011mixture} considers the setting of a mixture of manifolds. In this paper, we focus on one of the simplest manifold, a subspace,
	and consider the subspace clustering problem. Specifically, we approximate the original dataset as an union of subspaces. Representing the data as a union of subspaces allows for 
	more computationally efficient downstream analysis on various problems such as motion segmentation \citep{Vidal09}, handwritten digits recognition \citep{you2016divide}, and image compression \citep{hong2006multiscale}.
	
	\subsection{Related Work}
	Many techniques have been developed for subspace clustering, see \cite{elhamifar2013sparse} for a review. The mainstream methods usually include two phases: $(1)$ calculating the affinity matrix; $(2)$ applying spectral clustering \citep{ng2002spectral} to the affinity matrix to compute a label for each data point. For phase $(1)$, the property of self-representation is often used to calculate the affinity matrix: self-representation states that a point can be represented by a linear combination of other points in the same subspace. Specifically,  \cite{Vidal09} proposed the sparse subspace clustering (SSC) algorithm which solves the lasso minimization problem $N$ times, where $N$ is the total number of data points. Similarly, \cite{rahmani2017subspace} proposed the direction search algorithm (DSC) which uses $\ell_1$ minimization to find the ``optimal direction'' for each data point, these directions are then used to cluster the data points. One of the main drawbacks of SSC and DSC is their computational complexity of $O(N^2)$ in both time and space, which limits its application to  large datasets. To address this limitation, a variety of methods have been proposed to avoid solving complicated optimization problems in constructing the affinity matrix. \cite{heckel2015robust} used inner products with thresholding (TSC) to calculate the affinity between each pair of points, \cite{park2014greedy} used a greedy algorithm to find for each point the linear space spanned by its neighbors, similarly \cite{dyer2013greedy} and \cite{you2016scalable} used orthogonal matching pursuit (OMP), \cite{you2016oracle} used elastic the net for subspace clustering (ENSC) and proposed an efficient solver by active set method. However, these methods require running spectral clustering on the full  $N \times N$ affinity matrix. A Bayesian mixture model was proposed for subspace clustering in \cite{thomas2014learning}, but its parameter inference is not scalable to large dataset. \cite{zhou2018deep} used a deep learning based method which does not have theoretical guarantee. 
	
	Recently, there have been two methods that increase the scalability of sparse subspace clustering. The SSSC algorithm and its varieties \citep{Peng2015:SRSC_full} clusters a random subset of the whole dataset and then uses this clustering to classify or label the out-of-sample data points. This method scales well when the random subset is small, however a great deal of information is discarded as only the information in the subset is used. In  \cite{you2016divide} a divide and conquer strategy is used for SSC---the dataset is split into  several small subsets on which SSC is run, and clustering results are merged. This method cannot reduce the computational complexity of the SSC by an order of magnitude so is limited in its ability to scale to large dataset.

	\subsection{Contribution}
	In this paper, we propose a novel, efficient sampling based algorithm with provable guarantees that extends the ideas in previous scalable methods 
	\citep{Peng2015:SRSC_full,you2016divide}. The motivation for using sampling based algorithm is twofold. From the theoretical perspective, \cite{luxburg2005limits} showed that under certain assumptions, the spectral clustering results on the sampled subset will converge to the results on the whole dataset. This gives us the insight that as the size of the sampled subset increases properly, the subset becomes almost as informative as the whole dataset. From the computational perspective, traditional spectral clustering based algorithms need to build a ``neighborhood'' for each data point. Thus the complexity (both in time and memory) is usually at least quadratic in the total number of data points, while sampling based algorithms need to find neighboring points only within the subset. This greatly reduces the computational resources needed incurring some loss of information. 
	
	Our algorithm seeks to combine strengths from both approaches. In particular, for each point in the subset we find its nearest neighbors in the complete dataset and use these points to construct a sub-cluster, these sub-clusters contain information from the entire dataset and not just the random sample. Finding neighboring points among the whole dataset makes it possible to get a neighborhood with big enough size and few false connections for each sampled point. 
	The affinity matrix for the subset is then constructed from these sub-clusters. The idea is that we change the problem from ``clustering of data points'' to ``clustering of sub-clusters'', which integrates information across the dataset and should deliver better clustering results. 
	
	We provide theoretical guarantees for our procedure in Section~\ref{sec:theory}. The analysis reveals that under mild conditions, the subspaces can share arbitrarily many intersections as long as most of their principal angles are larger than a certain threshold. While our algorithm for finding neighboring points is similar to that of \cite{heckel2015robust}, the data generation model and assumptions underlying our theorems are different---we take into account the fact that after normalization the noisy terms will no longer follow a multivariate Gaussian distribution. While our work is originally designed for linear subspace clustering problems. The idea of clustering through sub-clusters can be easily extended to general clustering problems. 
	
	Finally, we study empirical properties of the proposed algorithm on both  synthetic and real-world datasets selected to have diverse sizes. We show that the clustering through sub-clusters algorithm is highly scalable and can significantly boost the clustering accuracy on both the subset and whole dataset.
	The advantage of our algorithm over other state-of-the-art algorithms changes from marginal to significant as the size of the dataset increases.
	
	\subsection{Paper Organization}
	The rest of this paper is organized as follows: in Section~\ref{sec:algo:algo}, we describe the implementation of our clustering procedure, in Section~\ref{sec:theory} we state the model setting and theoretical guarantees for our procedure and explain in some details the geometric and distributional intuitions underlying our procedure. The detailed proofs can be found in Appendix~\ref{sub:sec:prorof-thms}. In Section~\ref{sub:sec:num} we present numerical experiments and compare our method with other state-of-the-art methods, a comprehensive report of the numerical results can be found in Appendix~\ref{sub:sec:more-numerical}.

	\subsection{Notation}
	Unless specified otherwise, we use capital bold letter to denote data matrix, and corresponding lower bold letter to denote the columns of it. In this paper, we are given a dataset $\mathbf{Y}$ with $N$ data points in $\mathbb{R}^D$. We use both $\mathbf{y}_i$  and $[\mathbf{Y}]_i$ to denoted the $i$-th column of $\mathbf{Y}$, and $\mathbf{Y}_{-i}$ is the matrix $\mathbf{Y}$ with the $i$-th column removed. Similarly, we write $\mathbf{y}_{-j}$ as vector $\mathbf{y}$ with the $j$-th entry removed. The $ij$-th entry of a matrix $\mathbf{Y}$ is denoted as $[\mathbf{Y}]_{ij}$. The complement of event $\mathcal{E}$ is denoted by $\mathcal{E}^\complement$. The cardinality of $\mathcal{E}$ is denoted by $\mathbf{card}(\mathcal{E})$, and the mode of $\mathcal{E}$ is $\mathbf{mode}(\mathcal{E})$. We use subscript with parenthesis to represent the order statistics of entries in a vector, for example $\mathbf{a}_{(i)}$ is the $i$-th smallest entry in vector $\mathbf{a}$, while without ambiguity both $\mathbf{a}(i)$ and $a_i$ refer to the $i$-th element of vector $\mathbf{a}$. The unit sphere in $\mathbb{R}^d$ is denoted by $\mathbb{S}^{d-1}$. We assume each data point of $\mathbf{Y}$ lies on one of $K$ linear subspaces denoted by $\lbrace \mathcal{S}_k \rbrace_{k=1}^K$. Here $K$ is a known constant and $\mathcal{S}_k$ is the $k$-th linear subspace. The subspace clustering problem aims assigning to each point in $\mathbf{Y}$ the membership to a  subspace (cluster) $\mathcal{S}_k$. 
	
	We write $d_k$ as the dimension of subspace $\mathcal{S}_k$ and $\mathbf{U}_k \in \mathbb{R}^{D \times d_k}$ as its corresponding orthogonal base. The number of points belong to cluster $\mathcal{S}_k$ is $N_k$. We use $\mathbf{y_i}^{(k)} \in \mathbb{R}^D$ to represent the $i$-th point from the $k$-th cluster, the set $\lbrace \mathbf{y_1}^{(k)},...,\mathbf{y_{N_k}}^{(k)} \rbrace$ contains all points that belong to $\mathcal{S}_k$. Finally, we write $F_{m,n}$ as the $F$-distribution with parameters $(m,n)$, $Dir(\boldsymbol \alpha)$ as the Dirichlet distribution with parameter vector $\boldsymbol{\alpha}$, $\beta(a, b)$ as the beta distribution with parameters $(a,b)$, $\mathcal{N}(\mu, \sigma)$ as the Gaussian distribution with mean (vector) $\mu$ and variance (covariance matrix) $\sigma^2$, $\chi^2_d$ as the chi-square distribution with $d$ degrees of freedom, and $U(\mathbb{S}^{d-1})$ as the uniform distribution on the surface of unit sphere $\mathbb{S}^{d-1}$.
	
	\section{The Algorithm for Sampling Based Subspace Clustering}\label{sec:algo:algo}
	In this section, we introduce our sampling based algorithm for subspace clustering (SBSC). In Appendix~\ref{sec:algo:parameter} we will discuss issues regarding hyper-parameters. Throughout this section, we assume the columns of $\mathbf{Y}$ have unit $\ell_2$ norm.
	
	Our main algorithm takes the raw dataset $\mathbf{Y}$ and several parameters as inputs and outputs the clustering assignment for each point in the dataset, it proceeds in two 
	stages (see the matched steps in Algorithm~\ref{sub:algo:sbsc} for further details): 
	\begin{itemize}
		\item Stage 1: In-sample clustering
		\begin{enumerate}[label=\arabic*.]
			\item Draw a subset $\mathbf{\tilde{Y}}$ of $n \ll N$ points. 
			\item For each point $\mathbf{\tilde{y}}_i \in \mathbf{\tilde{Y}}$, find its $(d_{\max} + 1)$ nearest neighboring points in $\mathbf{Y}$ and use $\mathcal{C}_i$ to denote the index set of these points. We call $\mathbf{Y}_{\mathcal{C}_i}$ the sub-cluster of $\mathbf{\tilde{y}}_i$. 
			\item Compute the affinity matrix $\mathbf{D}$  where each element $[\mathbf{D}]_{ij}$ is the similarity calculated between $\mathbf{Y}_{\mathcal{C}_i}$  and $\mathbf{Y}_{\mathcal{C}_j}$. 
			\item Sparsify the affinity matrix by removing possible spurious connections.  
			\item Conduct spectral clustering on $\mathbf{\tilde{Y}}$ with the sparsified affinity matrix. 
		\end{enumerate}
		\item Stage 2: Out-of-sample classification
		\begin{enumerate}[label=\alph*]
			\item[6.] Fit a classifier to the clustered points in $\mathbf{\tilde{Y}}$ and classify the points in $\mathbf{{Y}} \setminus  \mathbf{\tilde{Y}}$. 
		\end{enumerate}
	\end{itemize}

	Step~2 computes a neighborhood of points around each sampled points by thresholding inner product similarities, the same method that was used in \cite{heckel2015robust}.
	The intuitive reason for this step is that for normalized data, two vectors are more likely to lie in the same linear subspace if the absolute magnitude of the  inner product between the points is large. One may use other measure of similarities in Step~2 to find the neighboring points. In addition to the standard algorithm, we also present experimental results based on other measure of similarities in Section~\ref{sub:sec:num} and Appendix~\ref{sub:sec:more-numerical}.  
	
	The idea of using distance between the sub-clusters to construct an affinity matrix in Step~3 relies on the  self-representative property of linear subspaces---see Theorem~\ref{sub:thm:cluster-dis-noisy} for technical details. 
	Please note that each entry of affinity matrix measures the closeness between data points, hence it decreases with distance function. There is both theoretical and empirical evidence that sparsification of an affinity matrix by setting smaller elements to zero improves clustering results \citep{belkin2003laplacian,von2007tutorial}. For this reason in Step~4 we threshold the affinity matrix.  Once the subset is clustered, the remaining points are labeled via a classifier where a regression model is fitted on the clustered data, specifically  a residual minimization model by ridge regression. If both $n$, $d_{max}$ and $D$ are linear in $\log N$, the complexity of our algorithm is $O(N\log N)$. 
	
	Note that any classifier can be used to do the out-of-sample classification. While ridge regression model is proved to work well for linear subspace clustering problems in this paper, we encourage users of Algorithm~\ref{sub:algo:sbsc} to choose their own favorite classifier, e.g., svm, random forest, or even deep neural networks, based on their understandings of the data.
	
	\IncMargin{1em}
	\begin{algorithm}[t]
		\caption{Sub-cluster Based Subspace Clustering (SBSC) algorithm.}\label{sub:algo:sbsc}
		\SetKwData{Left}{left}\SetKwData{This}{this}\SetKwData{Up}{up}
		\SetKwFunction{Union}{Union}\SetKwFunction{FindCompress}{FindCompress}
		\SetKwInOut{Input}{input}\SetKwInOut{Output}{output}
		\Input{Data $\mathbf{Y}$, number of subspaces $K$, sampling size $n$, neighbor threshold $d_{\max}$, regularization parameters $\lambda_1$ and $\lambda_2$, residual minimization parameter $m$, affinity threshold $t_{\max}$. }
		\Output{The label vector $\boldsymbol{\ell}$ of all points in $\mathbf{Y}$}
		1. Uniformly sample $n$ points  $\tilde{\mathbf{Y}}$ from $\mathbf{Y}$. \\ 
		
		2. Construct the sub-clusters: \\
		\For{$i$ = 1 \KwTo n}  { 
			$\mathbf{p}=| \langle \mathbf{\tilde{y}}_{i} ,\mathbf{Y} \rangle|$;\\
			$\mathcal{C}_i:= \lbrace j: | \langle \mathbf{\tilde{y}}_{i}, \mathbf{y}_j \rangle | \geq \mathbf{p}_{(N-d_{\max})} \rbrace$.}
		
		3. Construct affinity matrix $[\mathbf{D}]_{ij}=e^{-d(\mathbf{Y}_{\mathcal{C}_i},\mathbf{Y}_{\mathcal{C}_j})/2}$ for $i \neq j \in \lbrace 1,...,n \rbrace$ and
		\begin{eqnarray*}
			d(\mathbf{Y}_{\mathcal{C}_i},\mathbf{Y}_{\mathcal{C}_j}) & =& ||\mathbf{Y}_{\mathcal{C}_i}-\mathbf{Y}_{\mathcal{C}_j}(\mathbf{Y}_{\mathcal{C}_j}^T \mathbf{Y}_{\mathcal{C}_j}+\lambda_1 \mathbf{I})^{-1}\mathbf{Y}_{\mathcal{C}_j}^T \mathbf{Y}_{\mathcal{C}_i}||_F \\
			& & +||\mathbf{Y}_{\mathcal{C}_j}-\mathbf{Y}_{\mathcal{C}_i}(\mathbf{Y}_{\mathcal{C}_i}^T \mathbf{Y}_{\mathcal{C}_i}+\lambda_1 \mathbf{I})^{-1}\mathbf{Y}_{\mathcal{C}_i}^T \mathbf{Y}_{\mathcal{C}_j}||_F.
		\end{eqnarray*}
		
		4. Sparsify the adjacency matrix: \\
		\For  {$j=1$ \KwTo $n$} {
			$\mathbf{v}:=[\mathbf{D}]_{j}$; \\
			\For  {$i=1$ \KwTo $n$}  {
				\If {$[\mathbf{D}]_{ij} \leq \mathbf{v}_{(n-d_{\max})}$}    {$[\mathbf{D}]_{ij}:=0$}}}
		
		5.  Cluster $\tilde{\mathbf{Y}}$ : set $\mathbf{D} := \mathbf{D} + \mathbf{D}^T$ and cluster the in-sample points in $\tilde{\mathbf{Y}}$ by applying spectral clustering on $\mathbf{D}$, use $\boldsymbol{\ell}_{in}$ to denote the labels of $\tilde{\mathbf{Y}}$.\\
		6. Label the remaining points: use the Residual Minimization by Ridge Regression (RMRR) algorithm in Appendix~\ref{sub:sec:RMRR} to classify the remaining points in $\mathbf{Y} \setminus \tilde{\mathbf{Y}}$, specifically for the out-of-sample label we have 
		\[
		\boldsymbol{\ell}_{out} = RMRR(\mathbf{Y} \setminus \tilde{\mathbf{Y}}, \tilde{\mathbf{Y}}, \boldsymbol{\ell}_{in}, \lambda_2, m)
		\]
		7. Combine $\boldsymbol{\ell}_{in}$ and $\boldsymbol{\ell}_{out}$ to get $\boldsymbol{\ell}$, the label of the whole dataset $\mathbf{Y}$.
	\end{algorithm}

	\section{Clustering Accuracy}\label{sec:theory}

	\subsection{Model Specification}\label{sec:model_setting}
	We assume all subspaces have the same dimension $d$ and the data generating process is
	\begin{equation*}
	\hat{\mathbf{y_i}}^{(k)}=\zeta_i^{(k)} \mathbf{U_k} \mathbf{a}_i^{(k)}+\hat{\mathbf{e}}_i^{(k)},\quad i=1,\ldots,N_k, \quad k=1,...,K,
	\end{equation*}
	where $\mathbf{a}_i^{(k)} \in \mathbb{R}^{d}$ is sampled from the uniform distribution on the surface of $\mathbb{S}^{d-1}$,  $\zeta_i^{(k)}$ is a random scalar such that $ \zeta_i^{(k)2}\sim \chi^2_{d}$, and $\hat{\mathbf{e}}_i^{(k)} \sim \mathcal{N}(\mathbf{0},d \sigma^2 \mathbf{I}_D)$. 
	However $\hat{\mathbf{y_i}}^{(k)}$ are unobserved and we only observe the normalized version $\mathbf{y}_i^{(k)}=\hat{\mathbf{y_i}}^{(k)}/\|\hat{\mathbf{y_i}}^{(k)}\|_2$. We then have
	\begin{equation}\label{sub:eq:normalization}
	\mathbf{y}_i^{(k)} =  \frac{\mathbf{U_k} \mathbf{a}_i^{(k)}+\sigma\mathbf{e}_i^{(k)}}{\norm{\mathbf{U}_k \mathbf{a}_i^{(k)}+\sigma\mathbf{e}_i^{(k)}}_2}.
	\end{equation}
	Consequently, each entry in $\mathbf{e}_i^{(k)}=
	\hat{\mathbf{e}}_i^{(k)}/(\sigma \zeta_i^{(k)})$ follows multivariate $t$-distribution with $d$ degrees of freedom, and ${\|\mathbf{e}_i^{(k)}}_2^2 \|_2^2 / {D} \sim F_{D,d}$. Numerically, the normalizing constant $\| \mathbf{U}_k \mathbf{a}_i^{(k)}+\sigma\mathbf{e}_i^{(k)} \|_2$ will be approximately $1$. In \cite{heckel2015robust}, the normalizing constants are treated directly as $1$ and their noise vector is a multivariate Gaussian vector. In developing theoretical guarantees of this paper, we explicitly account for the normalizing constant $\| \mathbf{U}_k \mathbf{a}_i^{(k)}+\sigma\mathbf{e}_i^{(k)} \|_2$ and its effects. 
	
	Let $\lambda_{1}^{(ij)} \geq \lambda_{2}^{(ij)} \geq...\geq \lambda_{d}^{(ij)}$ correspond to the cosine values of principal angles between $\mathcal{S}_i$ and $\mathcal{S}_j$, hence $\lambda_{1}^{(ij)} \leq 1$ and $\lambda_{d}^{(ij)} \geq 0$. Note that $\lambda_{k}^{(ij)}=\lambda_{k}^{(ji)}$ for $1\leq k \leq d$ and $1 \leq i <  j \leq K$. For each subspace $\mathcal{S}_k$, we define the uniformly maximal affinity vector to quantify its closeness with respect to all other subspaces. 
	
	\begin{defn}\label{defn:max-affinity}
		For each subspace $\mathcal{S}_k$, its uniformly maximal affinity vector with respect to other subspaces is $[\lambda_{1}^{(k)},...,\lambda_{d}^{(k)}]$ such that
		\[
		\lambda_i^{(k)}=\max_{j \neq k} \lambda_{i}^{(kj)}.
		\] 
	\end{defn}

	\begin{defn}\label{defn:sub-preserve}
		We say a subspace clustering algorithm has sub-cluster preserving property if it produces sub-clusters $\lbrace \mathbf{Y}_{\mathcal{C}_i} \rbrace_{i=1}^n$ each  containing points from only one subspace.
	\end{defn}
	If the uniformly maximal affinity vectors have small entries, corresponding to large angles, we would expect that Algorithm~\ref{sub:algo:sbsc} (SBSC) to have sub-cluster preserving property.
	
	In constructing the affinity matrix $\mathbf{D}$, we want the following property: two sub-clusters that belong to the same subspace have bigger affinities, hence smaller distances, than sub-clusters that belong to different subspaces.
	\begin{defn}\label{sub:defn:right-neighbor}
		We say $\mathbf{Y}_{\mathcal{C}_i}$ has the correct neighborhood property with distance function $d(\cdot,\cdot)$ if 
		\begin{align*}
		d(\mathbf{Y}_{\mathcal{C}_i},\mathbf{Y}_{\mathcal{C}_j}) < d(\mathbf{Y}_{\mathcal{C}_i},\mathbf{Y}_{\mathcal{C}_k}),
		\end{align*}
		for any $1 \leq j \neq k \leq n$ such that $\mathbf{Y}_{\mathcal{C}_i}$ and $\mathbf{Y}_{\mathcal{C}_j}$ belong to the same subspace, and $\mathbf{Y}_{\mathcal{C}_k}$ belongs to a different subspace than $\mathbf{Y}_{\mathcal{C}_i}$.
	\end{defn}
	
	\subsection{Theoretical Properties of SBSC}
	\subsubsection{Assumptions}
	In this section, we list all the assumptions used by lemmas and theorems of Algorithm~\ref{sub:algo:sbsc}. Detailed proofs can be found in Appendix~\ref{sub:sec:prorof-thms}. Notice $A2$ subsumes $A1$, $A3$ includes $A1$, and $A4$ assumes $A1$, $A2$ and $A3$. The meaning of the assumptions is explained at the end of this section.
	
	\begin{enumerate}[label=A\arabic*.]
		\item There exist positive constants $T_l$ and $\rho$ such that
		\begin{align}
		&T_l^2 \leq \min_{k=1,...,K} Q_{1-\frac{d_{max}}{N_k^{1-\rho}}}, \label{sub:eq:beta-lower} 
		\end{align}
		where $Q_{p}$ denotes the $p$ quantile of $\beta(\frac{1}{2},\frac{d-1}{2})$.
		\item There exist positive constants $\lbrace g_{i} \rbrace_{i=1}^2$, $\eta$ and $\rho \in (0,1)$, such that if we write $T = \frac{4g_{2} + 2g_{2}^2}{1-g_{2}} +\frac{1+g_{2}}{1-g_{2}}g_{1}$, the following inequalities hold: \eqref{sub:eq:beta-lower} with $T_l$ replaced by $T$, and 
		\begin{align}
		&\sum_{i=1}^d \left(g_{1}^2-\lambda_i^{(k)2}\right)_+^2 > \sum_{i=1}^d  \left(g_{1}^2-\lambda_i^{(k)2}\right)_-^2, \quad \sum_{i=1}^d \left(g_{1}^2-\lambda_i^{(k)2}\right)_+ > \sum_{i=1}^d  \left(g_{1}^2-\lambda_i^{(k)2}\right)_-,  \label{sub:eq:assumption-A2-1} \\
		&\frac{g_2^2}{D \sigma^2} > 3+\frac{6}{\eta}, \quad \frac{d}{\log{N}} \geq (2 +2 \eta)^2. \label{sub:eq:assumption-A2-2}
		\end{align} 
		
		\item There exist positive constants $T_l$, $q_0$, $\rho$ and $t$ such that the following inequalities hold: \eqref{sub:eq:beta-lower}, $d_{max} > d$ and 
		\begin{equation}\label{sub:eq:assumption-A3}
		\frac{(T_l^2 d_{max}-C_2)C_2 - \frac{C_1^2}{4}}{T_l^2d_{max}} \geq q_0,
		\end{equation}
		where 
		\begin{align*}
		&C_1 =  \left(2 + t\sqrt{\frac{\log{N}}{d-2}} \right)\sqrt{d_{max}},\textit{ } C_2 =\left(\sqrt{\frac{2d_{max}}{\pi (d-1)}} - 2 - t \sqrt{\frac{\log{N}}{d-2}}\right)^2 \bigg/ 2.
		\end{align*}
		\item There exist positive constants $T_l$, $g$, $\lambda$, $\eta$, $q_0$, $\rho$ and $t$ such that the following inequalities hold: \eqref{sub:eq:beta-lower}, \eqref{sub:eq:assumption-A2-1} with $g_1$ replaced by $T_l$, \eqref{sub:eq:assumption-A2-2} with $g_2$ replaced by $g$, \eqref{sub:eq:assumption-A3} and
		\begin{equation}\label{sub:eq:assumption4}
		\begin{aligned}
		&f(d):= \frac{(2g-g^2)\left(d_{max}+1\right)}{2(1-g)} \cdot \sqrt{\frac{d(1+g)^4}{q_0^2} + \frac{D-d}{\lambda^2}} \leq \frac{1}{2}, \\ 
		& \frac{f(d)\sqrt{d_{max}+1}}{1-f(d)} \cdot \sqrt{\frac{d(1+g)^4\lambda^2}{q_0^2} + D-d} \leq \frac{\lambda (1+g)^2 \sqrt{d(d_{max}+1)}}{q_0(1-g)}, \\
		& g\sqrt{\frac{d(1+g)^4\lambda^2}{q_0^2} + D-d} \leq \frac{\lambda (1+g)^2 \sqrt{d(d_{max}+1)}}{q_0(1-g)}, \\
		&\frac{6\lambda (1+g)^2 \sqrt{d(d_{max}+1)}}{q_0(1-g)} \leq \sqrt{1-T_l^2}.
		\end{aligned}
		\end{equation}
	\end{enumerate}
	
	Assumption $A1$ is used to bound the order statistics of a Beta distribution in Lemma~\ref{sub:lem:beta-bound}. It implicitly controls the ratio between $d$ and $\log{N}$. If we write $N_k = 10000$, $N = 10N_k$, $d_{max}=3d$, $T_l^2 = 0.09$, and $\rho=0.01$, then it suffices to have $\frac{d}{\log{N_k}} \leq 5$ for inequality \eqref{sub:eq:beta-lower}. 
	
	Assumption $A2$ is the subspace separation assumption. We use it for the proof of Theorem~\ref{sub:thm:sub-preserve-noisy}. In Appendix~\ref{sub:sec:prorof-thms}, we show that SBSC requires most of ${\lbrace \lambda_i^{(k)} \rbrace}_{i=1,k=1}^{d, \ \   K}$ to be smaller than $g_1$. This means large $g_1$ implies an easier clustering problem for SBSC, and vice versa. Throughout this paper we call $g_1$ the affinity threshold. Note that $T$ is a upper bound of the affinity threshold $g_1$, specifically if there was no noise $T = g_1$. From \eqref{sub:eq:beta-lower} we know that large $\frac{d}{\log N}$ implies a small $T$ and $g_1$. Therefore, large $d$ makes the clustering problem harder. This agrees with our intuition. Consider the extreme case where the subspaces are orthogonal, $\lambda_{i}^{(k)}=0\  (i=1,\ldots, d, k=1,\ldots, K)$, and Equations~\eqref{sub:eq:assumption-A2-1} are naturally true with any positive constant $g_1$. Finally, the constant $g_2$ in \eqref{sub:eq:assumption-A2-2} controls the noise term. From the first condition in  \eqref{sub:eq:assumption-A2-2} we have $\sigma < \frac{g_2}{\sqrt{D}}$. 
	
	Assumption $A3$ guarantees the sub-clusters $\lbrace \mathbf{Y}_{\mathcal{C}_i} \rbrace_{i=1}^n$ are informative. We use it mainly for the proof of Lemma~\ref{sub:lem:informative}. Here $C_1$ and $C_2$ are closely related to the permeance statistics \citep{lerman2012robust}, which measures how well a set of vectors is scattered across a space. Therefore a large $\frac{d_{max}}{d}$ implies that these vectors are well scattered. If $T_l^2$ equals to its upper bound in \eqref{sub:eq:beta-lower}, $\rho = 0.01$, $N_k=100000$, $N = 10N_k$, $d_{max} = 160d$, $t=0.05$ and we want $q_0 \geq 0.5$, $A3$ requires $\frac{d}{\log N} \leq 5$ \footnote{In this example, $\frac{d_{max}}{d}$ is fairly large. In the numerical section we found it is usually not necessary to choose large $d_{max}$. A better bound in Corollary~\ref{sub:coro:permeance} might be helpful the bridge the gap between numerical experiments and theoretical guarantee.}. 
	
	Assumption $A4$ is a combination of all previous assumptions, with slightly stronger conditions on subspace similarities and noise level; we use it for the proof of Theorem~\ref{sub:thm:cluster-dis-noisy}. The first three conditions in \eqref{sub:eq:assumption4} essentially control the value of $g$, which in turn controls the magnitude of the norm of noise terms. The last condition in \eqref{sub:eq:assumption4} controls the value of the regularization parameter in a distance function that will be defined latter.
	
	\subsubsection{Theoretical Properties of SBSC}
	Two theorems regarding the State $1$ of SBSC are discussed in this section.
	
	\begin{thm}\label{sub:thm:sub-preserve-noisy}
		Under Assumption $A2$, SBSC has sub-cluster preserving property with probability at least
		\begin{align} \label{sub:eq:thm-1-main}
		1- \sum_{k=1}^K  \frac{n_k(N_k-d_{\max})}{d_{\max}(N_k+1)(N_k^{\rho}-1)^2}-2(K-1)n e^{-\epsilon_1^2}-\frac{2N}{N^{\left(1 + \frac{\eta}{2+\eta}\right)^2}},
		\end{align}
		where 
		\begin{align}\label{sub:eq:thm1-epsilon}
		\epsilon_1 = \min_{k} \frac{\sum_{i=1}^d \left(g_{1}^2-\lambda_i^{(k)2}\right)_+ - \sum_{i=1}^d  \left(g_{1}^2-\lambda_i^{(k)2}\right)_-}{2 \sqrt{\sum_{i=1}^d \left(g_{1}^2-\lambda_i^{(k)2}\right)_+^2}+\sqrt{4 \sum_{i=1}^d \left(g_{1}^2-\lambda_i^{(k)2}\right)_+^2+2 \sum_{i=1}^d \left(g_{1}^2-\lambda_i^{(k)2}\right)_+}}.
		\end{align}
	\end{thm}
	
	If the subspaces are orthogonal with each other, i.e. $\lbrace \lambda_i^{(k)} \rbrace_{i=1, k = 1}^{d,\quad K} = 0$. Equation \eqref{sub:eq:thm1-epsilon} is
	\begin{align*}
	\epsilon_1 = \frac{\sqrt{d}}{2 + \sqrt{4 + \frac{2}{g_1^2}}}.
	\end{align*}
	This shows $\epsilon_1$ is linear in $\sqrt{d}$ and monotonically increasing in $g_1^2$. Appendix~\ref{sub:sec:more-technical} establishes general conditions on $g_1$ and $\lbrace \lambda_i^{(k)} \rbrace_{i=1, k = 1}^{d,\quad K}$ under which $\epsilon_1$ grows like $\sqrt{d}$. Combining this with Assumption $A2$, we observe that the third term of \eqref{sub:eq:thm-1-main} is small for large $N$. 
	
	Next, we use the sub-cluster preserving property established in Theorem~\ref{sub:thm:sub-preserve-noisy} to prove the theoretical guarantee for correct neighborhood property (see Definition~\ref{sub:defn:right-neighbor}). We define a distance function between two sub-clusters as
	\begin{equation}\label{sub:eq:distance-function}
	d(\mathbf{Y}_{\mathcal{C}_i},\mathbf{Y}_{\mathcal{C}_j})= \|\mathbf{Y}_{\mathcal{C}_i}-\mathbf{Y}_{\mathcal{C}_j}(\mathbf{Y}_{\mathcal{C}_j}^T \mathbf{Y}_{\mathcal{C}_j}+\lambda \mathbf{I})^{-1}\mathbf{Y}_{\mathcal{C}_j}^T\mathbf{Y}_{\mathcal{C}_i}\|_F +\|\mathbf{Y}_{\mathcal{C}_j}-\mathbf{Y}_{\mathcal{C}_i}(\mathbf{Y}_{\mathcal{C}_i}^T \mathbf{Y}_{\mathcal{C}_i}+\lambda \mathbf{I})^{-1}\mathbf{Y}_{\mathcal{C}_i}^T\mathbf{Y}_{\mathcal{C}_j}\|_F,
	\end{equation}
	where $\lambda>0$ is a regularization parameter.
	
	\begin{thm}\label{sub:thm:cluster-dis-noisy}
		Assume sub-cluster preserving property is true for SBSC with probability at least $1-p_s$, and Assumption $A4$ is satisfied. Then $\lbrace \mathbf{Y}_{\mathcal{C}_i} \rbrace_{i=1}^n$ have the correct neighborhood property with the distance function \eqref{sub:eq:distance-function} with probability at least
		\begin{align*}
		1 &- 4n(n-1)e^{-\epsilon_1^2} -\frac{2n}{N^{t^2/2}} - \sum_{k=1}^K n_k \left( \frac{N_k-d_{max}}{d_{max}(N_k+1)(N_k^{\rho}-1)^2} +2(N_k-1) e^{-{\epsilon_2^2}} \right) \nonumber \\
		& -\frac{2N}{N^{\left(1 + \frac{\eta}{2+\eta}\right)^2}} -p_s,
		\end{align*}
		where $\epsilon_1$ is defined in \eqref{sub:eq:thm1-epsilon} with $g_1$ replaced by $T_l$ and
		\begin{equation}\label{sub:eq:beta_upper}
		\epsilon_2 =  \frac{\sqrt{d-1}-1}{2 + \frac{1}{\sqrt{d-1}+1}}.
		\end{equation}
	\end{thm}
	
	\section{Experimental Results}\label{sub:sec:num}
	In this section, we test the performance of SBSC on both synthetic and benchmark datasets. In addition to Algorithm~\ref{sub:algo:sbsc}, we also consider two modifications of the SBSC algorithm. The SBSC-DSC algorithm uses optimal direction search algorithm \citep{rahmani2017subspace}
	instead of high correlation to find neighboring points in Step 2 of Algorithm~\ref{sub:algo:sbsc}. The SBSC-SSC algorithm uses lasso minimization in Step 2 of Algorithm~\ref{sub:algo:sbsc}; its numerical results are reported in Appendix~\ref{sub:sec:more-numerical}.

	The performance of the three versions of SBSC is compared to other state-of-the art algorithms. These include classic subspace clustering method: Sparse Subspace Clustering (SSC, \citealp{Vidal09, you2016divide}), Thresholding Subspace Clustering (TSC, \citealp{heckel2015robust}), Direction Search Subspace Clustering (DSC, \citealp{rahmani2017subspace}), Least Square Regression (LSR, \citealp{lu2012robust}), Low-Rank Representation (LRR, \citealp{liu2010robust}), Subspace Clustering by Orthogonal Matching Pursuit (SSC-OMP, \citealp{you2016scalable}), Elastic Net Subspace Clustering (ENSC, \citealp{you2016oracle}); and sampling based algorithms \citep{peng2013scalable}: Scalable Sparse Subspace Clustering (SSSC, reported in Appendix~\ref{sub:sec:more-numerical}), Scalable Thresholding Subspace Clustering (STSC), Scalable Direction Search (SDSC), Scalable Least Square Regression (SLSR), Scalable Low-Rank Representation (SLRR). To make fair comparisons, where possible we replicated results on our machine. Some results had to be copied from the original paper due to the unavailability of code. 
	
	Throughout this section, we use clustering accuracy, normalized mutual information (NMI) and running time as the metrics for performance evaluation. The formulas for clustering accuracy and NMI are presented in Appendix~\ref{sub:sec:metric}. To demonstrate the advantage of using sub-clusters (i.e. borrowing information from the whole dataset) to cluster the data points in the subset. For sampling based algorithms we also report their clustering accuracy on the subset. In the rest of this paper, we call the clustering accuracy on the whole dataset as accuracy, and the clustering accuracy on the subset as accuracy-sub. For randomized algorithms, reported results are averaged over $10$ trials. Additional numerical results are presented in Appendix~\ref{sub:sec:more-numerical}. The code used to generate these results can be found in the supplementary material. 
	
	\subsection{Results on Synthetic Dataset}
	In this section we evaluate tolerance to noise and scalability on synthetic data generated using
	the model specified in Section~\ref{sec:model_setting}. 
	
	\subsubsection{Tolerance to Noise}\label{sub:sec:toltonoise}
	In this section, we test the tolerance to noise of the various algorithms. From \eqref{sub:eq:normalization} we know the un-normalized signal part $\mathbf{U}_k \mathbf{a}_i^{(k)}$ has unit $\ell_2$ norm, and we can calculate the expected squared norm of the noise part is $\mathbb{E} [ \sigma^2 \| \mathbf{e}_i^{(k)} \|_2^2 ] = {\sigma^2 Dd}/{(d-2)}$. Therefore throughout this paper we define ${(d-2)}/{(\sigma^2 Dd)} $ as the signal strength (signal to noise ratio). The noise captured the amount of variation of points in $\mathbb{R}^D$.
	
	We change the signal strength from $10$ to $2$. For each value of signal strength, we simulate $10$ datasets with  $K = 20$ subspaces, where each subspace contains $N_i = 10000$ data points. For all the sampling based algorithms we fixed the sampling size as $n = 200$. 
	
	\begin{figure}[t]
		\centering
		\subfloat[]{{\includegraphics[width=7cm,height=7cm]{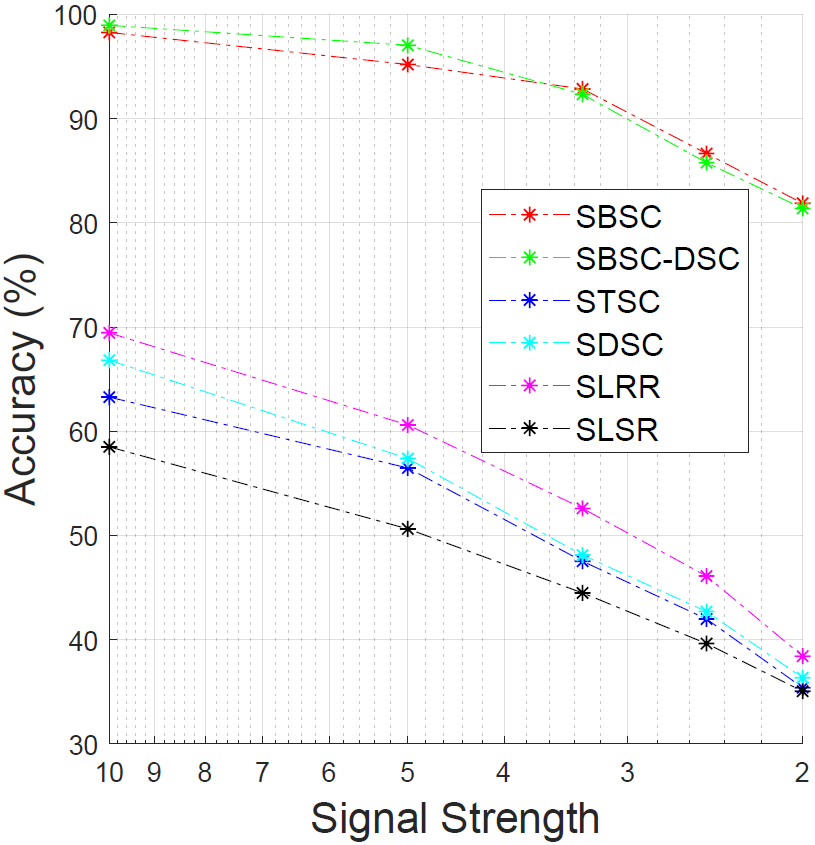} }}
		\qquad
		\subfloat[]{{\includegraphics[width=7cm,height=7cm]{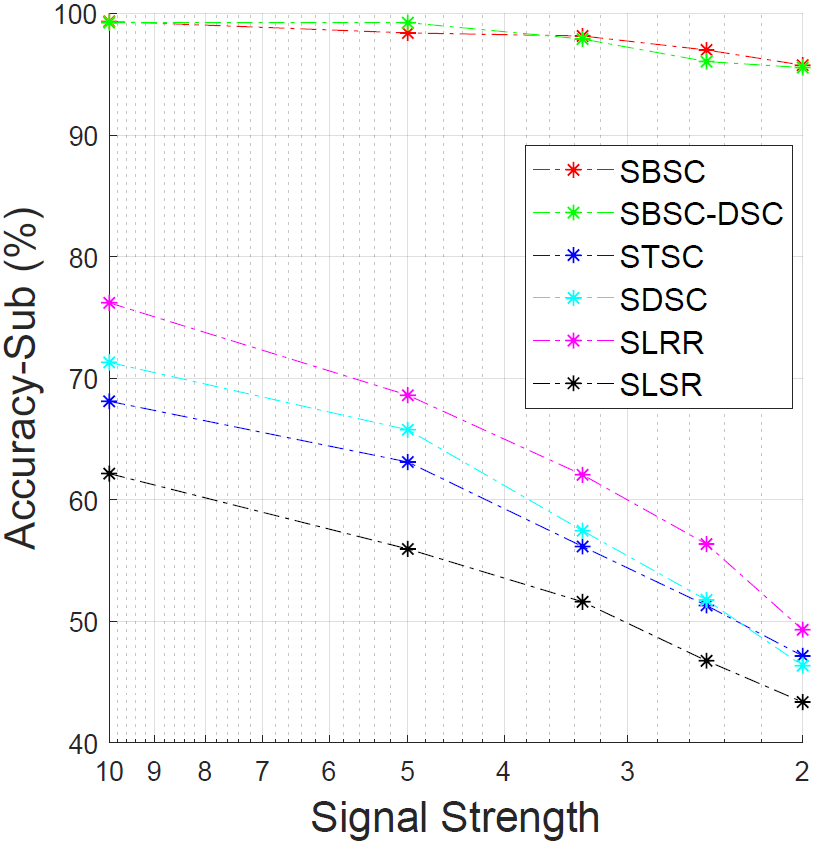} }}
		\caption{\textbf{Tolerance to Noise:} Plot of accuracy (left panel) and accuracy on subsets (right panel) for algorithms applied to synthetic datasets of Section~\ref{sub:sec:toltonoise}. The $x$-axis is the signal strength and the $y$-axis is the accuracy averaged over $10$ different datasets. SBSC performs very well.}%
		\label{fig:tol_to_noise}%
	\end{figure}
	
	The results are presented in Figure~\ref{fig:tol_to_noise}. Accuracy and accuracy-sub are plotted on the left and right hand side panels respectively. The small discrepancy between two sides shows both sampling based algorithms can deliver consistent results between in-sample clustering and out-of-sample classification. At the same time, the SBSC based algorithms constantly deliver much higher accuracy-sub than the other sampling based algorithms, this means for the synthetic datasets, borrowing information from the whole dataset significantly enhanced the clustering results for subset.    
	
	\subsubsection{Scalability}\label{sub:sec:scalability}
	In this section, we test the scalability of SBSC. Specifically, we randomly generate $K = 20$ subspaces in an ambient space with dimension $D = 30$, each of the subspaces has dimension $d = 5$. We increase $N_k$ from $100$ to $51200$, so the corresponding $N$ increases from $2000$ to $1024000$. The sampling size $n$ is $\floor{2K\log(N)}$. 
	
	The result is presented in Figure~\ref{fig:runtime}. On the right hand side y-axis, we show the average accuracy, which is around $95 \% $ across all experiments, against the number of data points $N$, this could justify our choice of $n$. On the left hand side y-axis, we show the scale plot between running time and $N$, the linear pattern here agrees with our complexity analysis. As we increase the number of data points $N$, the accuracy on the whole dataset slightly gets higher, this implies our algorithm is particularly useful for large datasets.
	
	\begin{figure}[t]
		\includegraphics[width=9cm]{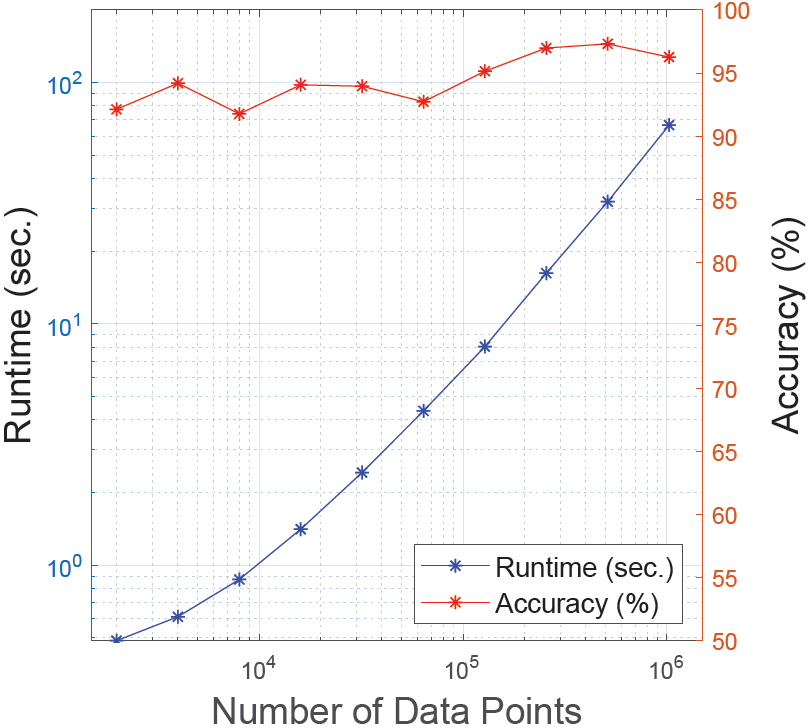}
		\centering
		\caption{\textbf{Scalability:} Plot of running time  (blue curve) and accuracy (red curve), averaged over $10$ independent datasets, versus the number of data points for SBSC applied to the data of Section~\ref{sub:sec:scalability}. Shows that the algorithm scales well.}
		\label{fig:runtime}
	\end{figure}
	
	\subsection{Results on Benchmark Datasets}
	In this section, we test SBSC on three benchmark datasets. These datasets were selected to have small, medium and large data size respectively. As expected, the advantage of SBSC over other state-of-the-art algorithms changes from marginal to significant as the size of the dataset increases.
	
	\subsubsection{The Extended Yale B dataset}
	The Extended Yale B dataset (YaleB) contains $N = 2432$ face images of $K = 38$ individuals. Each image is a front view photo of the corresponding individual with different illumination condition. To speed up the running time, a dimension reduction step is taken to pre-process the dataset \citep{rahmani2017subspace}, hence in our experiment $D = 500$.
	
	From Table~\ref{tab:yale}, we see that the DSC algorithm delivered the highest accuracy and NMI. As expected, for this small dataset sampling based algorithms did not perform as well. The reason is, there is not enough observations to form sufficiently large homogeneous sub-clusters. This issue is worse for datasets with large number of clusters.

	\begin{table}[h]
		\centering
		\begin{tabular}{|c|c|c|c|c|}\hline
			\textbf{Method}    & \textbf{Accuracy (\%)}         & \textbf{Accuracy-Sub (\%)}                                     & \textbf{NMI (\%)}                                                  & \textbf{Runtime (sec.)} \\ \hline
			SBSC  & \begin{tabular}[c]{@{}c@{}}26.53  \\ (1.8)\end{tabular} & \begin{tabular}[c]{@{}c@{}} 31.56\\ (1.94)\end{tabular}   & \begin{tabular}[c]{@{}c@{}} 41.67 \\ (1.44)\end{tabular} & 27    \\ \hline
			
			SBSC-DSC  & \begin{tabular}[c]{@{}c@{}}60.46  \\ (1.62)\end{tabular} & \begin{tabular}[c]{@{}c@{}} \textbf{62.76}\\ \textbf{(1.88)}\end{tabular}   & \begin{tabular}[c]{@{}c@{}} 70.15 \\ (0.97)\end{tabular} & 35   \\ \hline
			
			STSC  & \begin{tabular}[c]{@{}c@{}}17.45  \\ (1.27)\end{tabular} & \begin{tabular}[c]{@{}c@{}} 21.72\\ (1.7)\end{tabular}   & \begin{tabular}[c]{@{}c@{}} 29.17 \\ (1.28)\end{tabular} & \textbf{0.8}   \\ \hline
			
			SDSC  & \begin{tabular}[c]{@{}c@{}}52.18  \\ (1.96)\end{tabular} & \begin{tabular}[c]{@{}c@{}} 60.78\\ (2.04)\end{tabular}   & \begin{tabular}[c]{@{}c@{}} 54.61\\ (1.85)\end{tabular} & 3 \\ \hline
			
			SLRR  & \begin{tabular}[c]{@{}c@{}}18.35  \\ (0.64)\end{tabular} & \begin{tabular}[c]{@{}c@{}} 28.6\\ (1.64)\end{tabular}   & \begin{tabular}[c]{@{}c@{}} 26.33\\ (0.57)\end{tabular} & 7  \\ \hline
			
			SLSR  & \begin{tabular}[c]{@{}c@{}}26.48  \\ (1.98)\end{tabular} & \begin{tabular}[c]{@{}c@{}} 37.78\\ (2.33)\end{tabular}   & \begin{tabular}[c]{@{}c@{}} 35.21\\ (2.22)\end{tabular} & 1.5 \\ \hline
			
			TSC  & 26.19 &  NA  & 39.31 & 1 \\ \hline
			
			DSC  & \textbf{91.69} &  NA  & \textbf{93.43} &45 \\ \hline
			
			
			
			SSC  & 52.96 &  NA  & 60.15 & 169 \\ \hline
			
			SSC-OMP  & 73.88 &  NA  & 80.1 & 1.51 \\ \hline
			
			SSC-ENSC  & 60.81 &  NA  & 69.4 & 3 \\ \hline
		\end{tabular}
		\caption{\textbf{Performance on Extended Yale B:} The results of sampling based algorithms are averaged over $10$ independent runs and the corresponding standard deviations are presented in  parentheses. The remaining algorithms do not use random subsampling and were run only once.  The metric reported as ``NA'' is not defined for these algorithms.
			The best result of each performance metric is in {\bf bold}.
			DSC delivers the highest accuracy and NMI.}
		\label{tab:yale}
	\end{table}
	
	\subsubsection{The Zipcode dataset}
	The Zipcode dataset \citep{le1990handwritten} is a medium-size dataset with $N =  9298$ data points and $D = 256$, each point represents an image of handwritten digit, hence $K = 10$. 
	
	From Table~\ref{tab:zipcode} we see, that SBSC delivers the best results in all metrics except running time. However the differences in running time are marginal for sampling based algorithms. The accuracy-sub of SBSC is again better than that of traditional sampling based algorithms (see SBSC versus STSC, and SBSC-DSC versus SDSC). 
	
	
	\begin{table}
		\centering
		\begin{tabular}{|c|c|c|c|c|}
			\hline
			\textbf{Method}    & \textbf{Accuracy (\%)}         & \textbf{Accuracy-Sub (\%)}                                     & \textbf{NMI (\%)}                                                  & \textbf{Runtime (sec.)} \\ \hline
			SBSC  & \begin{tabular}[c]{@{}c@{}} \textbf{69.4}  \\ \textbf{(5.17)}\end{tabular} & \begin{tabular}[c]{@{}c@{}} \textbf{72.04}\\ \textbf{(5.37)}\end{tabular}   & \begin{tabular}[c]{@{}c@{}} 70.3 \\ (1.75)\end{tabular} & 10  \\ \hline
			
			SBSC-DSC  & \begin{tabular}[c]{@{}c@{}}60.84  \\ (2.87)\end{tabular} & \begin{tabular}[c]{@{}c@{}} 64.92\\ (3.37)\end{tabular}   & \begin{tabular}[c]{@{}c@{}} 62.92 \\ (0.65)\end{tabular} & 71   \\ \hline
			
			STSC  & \begin{tabular}[c]{@{}c@{}}55.28  \\ (4.25)\end{tabular} & \begin{tabular}[c]{@{}c@{}} 60.86\\ (3.8)\end{tabular}   & \begin{tabular}[c]{@{}c@{}} 53.1 \\ (2.49)\end{tabular} & \textbf{2}   \\ \hline
			
			SDSC  & \begin{tabular}[c]{@{}c@{}}45.62  \\ (6.43)\end{tabular} & \begin{tabular}[c]{@{}c@{}} 51.16\\ (7.31)\end{tabular}   & \begin{tabular}[c]{@{}c@{}} 45.99\\ (3.88)\end{tabular} & 3  \\ \hline
			
			SLRR  & \begin{tabular}[c]{@{}c@{}}63.21 \\ (3.96)\end{tabular} & \begin{tabular}[c]{@{}c@{}} 65.16\\ (4.03)\end{tabular}   & \begin{tabular}[c]{@{}c@{}} 66.09\\ (1.39)\end{tabular} & 10  \\ \hline
			
			SLSR  & \begin{tabular}[c]{@{}c@{}}58.66 \\ (0.99)\end{tabular} & \begin{tabular}[c]{@{}c@{}} 59.85\\ (0.98)\end{tabular}   & \begin{tabular}[c]{@{}c@{}} 62.54\\ (1.38)\end{tabular} & 4\\ \hline
			
			TSC  & 65.73 &  NA  & \textbf{78.97} & 115 \\ \hline
			
			DSC  & 60.92 &  NA  & 68.43 & 800 \\ \hline
			
			SSC  & 48.16 &  NA  & 52.37 & 2165 \\ \hline
			
			SSC-OMP  & 23.58 &  NA  & 25.51 & 3 \\ \hline
			
			SSC-ENSC  & 44.65 &  NA  & 50.08 & 36 \\ \hline
		\end{tabular}
		\caption{\textbf{Performance on Zipcode:} See the caption of Table~\ref{tab:yale} for description. We see SBSC delivers the highest accuracy and accuracy-sub, while TSC delivers the highest NMI.}
		\label{tab:zipcode}
	\end{table}

	\subsubsection{The MNIST dataset}
	The MNIST dataset (MNIST) contains $N=70000$ data points, each point represents an image of handwritten digit. The original data was transferred into $\mathbb{R}^{500}$ by convolutional neural network and PCA \citep{you2016scalable}. Again $K=10$.  
	
	From Table~\ref{tab:mnist} we see, that SBSC with bagging described in Appendix~\ref{sub:sec:bagging} dominates in nearly every aspect. The large data size of MNIST makes the sampling based algorithms run much faster than traditional methods. Due to their slow speed we did not use bagging on non-sampling algorithms.
	
	\begin{table}[h]
		\centering
		\begin{tabular}{|c|c|c|c|c|}
			\hline
			\textbf{Method}    & \textbf{Accuracy (\%)}         & \textbf{Accuracy-Sub (\%)}                                     & \textbf{NMI (\%)}                                                  & \textbf{Runtime (sec.)} \\ \hline
			SBSC(1) & \begin{tabular}[c]{@{}c@{}} 95.74 \\ (0.28) \end{tabular} & \begin{tabular}[c]{@{}c@{}} \textbf{96.44}\\ \textbf{(1.14)}\end{tabular}   & \begin{tabular}[c]{@{}c@{}} 89.9 \\ (0.47)\end{tabular} & 38   \\ \hline
			
			SBSC(6)  & \begin{tabular}[c]{@{}c@{}} \textbf{97.15}  \\ \textbf{(0.16)}\end{tabular} & \begin{tabular}[c]{@{}c@{}} 95.25\\ (1.78)\end{tabular}   & \begin{tabular}[c]{@{}c@{}} \textbf{92.6} \\ \textbf{(0.3)}\end{tabular} & 246    \\ \hline
			
			STSC(1)  & \begin{tabular}[c]{@{}c@{}}30.2  \\ (2.13)\end{tabular} & \begin{tabular}[c]{@{}c@{}} 67.8\\ (3.95)\end{tabular}   & \begin{tabular}[c]{@{}c@{}} 11.52 \\ (2.12)\end{tabular} & \textbf{28}   \\ \hline
			
			STSC(6)  & \begin{tabular}[c]{@{}c@{}}40.12  \\ (2.84)\end{tabular} & \begin{tabular}[c]{@{}c@{}} 65.23\\ (2.2)\end{tabular}   & \begin{tabular}[c]{@{}c@{}} 22.53 \\ (2.36)\end{tabular} & 172   \\ \hline
			
			SLRR(1)  & \begin{tabular}[c]{@{}c@{}}79.5 \\ (1.19)\end{tabular} & \begin{tabular}[c]{@{}c@{}} 79.46\\ (1.3)\end{tabular}   & \begin{tabular}[c]{@{}c@{}} 79.9\\ (1.52)\end{tabular} & 59  \\ \hline
			
			SLRR(6)  & \begin{tabular}[c]{@{}c@{}} 81 \\ (0.67)\end{tabular} & \begin{tabular}[c]{@{}c@{}} 79.6\\ (0.44)\end{tabular}   & \begin{tabular}[c]{@{}c@{}} 83.75\\ (0.74)\end{tabular} & 378  \\ \hline
			
			SLSR(1)  & \begin{tabular}[c]{@{}c@{}}75.06 \\ (6.11)\end{tabular} & \begin{tabular}[c]{@{}c@{}} 74.62\\ (5.99)\end{tabular}   & \begin{tabular}[c]{@{}c@{}} 76.21\\ (3.63)\end{tabular} & 54 \\ \hline
			
			SLSR(6)  & \begin{tabular}[c]{@{}c@{}}79.64 \\ (0.85)\end{tabular} & \begin{tabular}[c]{@{}c@{}} 76.43\\ (1.85)\end{tabular}   & \begin{tabular}[c]{@{}c@{}} 81.24\\ (0.95)\end{tabular} & 326 \\ \hline
			
			TSC  & 84.63 &  NA  & 87.47 & 1184 \\ \hline
			
			SSC (DC1)$^*$       & 96.55                  & NA                                      & NA                                                         & 5254                    \\ \hline
			SSC (DC2)$^*$         & 96.1                      & NA                                   & NA                                                         & 4390                    \\ \hline
			SSC (DC5)$^*$         & 94.9                 & NA                                        & NA                                                         & 1596                    \\ \hline
			
			SSC-OMP  & 81.51  &  NA  & 84.45 & 232 \\ \hline
			
			SSC-ENSC  & 93.79 &  NA  & 88.8 & 500 \\ \hline
		\end{tabular}
		\caption{\textbf{Performance on MNIST:} 
			See description in Table~\ref{tab:yale}.
			Additionally, the results of methods with star marks are copied from the original paper that did not report NMI. The number in the parenthesis next to the algorithm name is the number of bags. We see SBSC dominates other algorithms in nearly every aspect.}
		\label{tab:mnist}
	\end{table}

	\section{Conclusion}
	While the idea of subsampling was discussed  before \citep{Peng2015:SRSC_full}, the main contribution of this paper is finding neighborhood points in the whole dataset and using cluster-wise distance to cluster points in the subset. This results in a higher clustering accuracy.   
	
	In calculating cluster-wise distances and classifying out-of-sample points, ridge regression seems to be the most direct method. However, the algorithm itself is highly flexible. Users are encouraged to try different distance functions, classification methods, and metrics in finding neighboring points. 
	
	\nocite{*}
	\appendix
	\section{Practical Recommendations for Parameter Setting}\label{sec:algo:parameter}
	In Algorithm~\ref{sub:algo:sbsc} (SBSC), we assume the number of clusters is known. Several methods have been developed for the estimation of the number of clusters from data \citep{ng2002spectral}. Intuitively, $n$ should be large enough to represent the structure of the whole dataset while still being relatively small to reduce the computational complexity. In our numerical experiments, we choose $n$ to be linear in $K \log N$. 
	
	Ideally, each sub-cluster $\mathbf{Y}_{\mathcal{C}_i}$ should well represent the subspace it belongs to, i.e., contain at least one basis of that subspace. Therefore we want $d_{\max}$ to be larger than $\max_{k=1,...,K} d_k$ which is unknown. For this reason we set $d_{\max}$ to grow linearly with $D$. Similarly the residual minimization parameter $m$ should also be linear in $D$. 
	
	\subsection{Threshold Selection}
	The spectral clustering algorithm can deliver exact clustering result \citep{von2007tutorial} if the graph induced by the affinity matrix $(\mathbf{D}+\mathbf{D}^T)$  has no false connections; and has exactly $K$ connected components. For a large threshold parameter $t_{\max}$ on the affinity matrix more entries in $\mathbf{D}$ will be kept and our algorithm is more likely to have false connections, while small $t_{\max}$ eliminates false connections but might incur non-connectivity. 
	
	Let us consider a heuristic situation: the subset we sampled contains exactly the same points (hence $\frac{n}{K}$ points) for each cluster. Then if we choose the threshold index $t_{\max}$ to be $\frac{n}{2K}$, the induced graph from our affinity matrix will have no false connection (given that points from same subspace have bigger similarities between each other) and the clusters themselves will be connected, therefore the spectral clustering algorithm will deliver the exact clustering result \citep{luxburg2005limits}.
	
	In reality clusters do not usually have same points in $\hat{\mathbf{Y}}$, hence we choose $t_{\max}$ to start from a relatively large number $\frac{n}{0.5K}$ and gradually increase it. Based on different threshold values, we can generate different label vectors on the subset $\hat{\mathbf{Y}}$, intuitively label vectors that can deliver highly accurate results should be similar to each other or stable. Based on this intuition, we developed a simple adaptive algorithm for finding  an ``optimal'' affinity threshold $t_{\max}$; see supplementary code for details. Based on our observation, choosing $t_{max}$ adaptively works well with datasets where each cluster has large amount of points. 
	
	\subsection{Combining Runs of the Algorithm}\label{sub:sec:bagging}
	Thanks to the speed of our algorithm, we can conduct several independent runs for one experiment (for sampling based algorithms, the results between independent runs might be different) with acceptable running time. In order to make full use of such advantage, we designed an algorithm to combine the results via bagging from several runs of SBSC \citep{breiman1996bagging}. Unlike the classification problem, we need to conduct label switching, see Algorithm~\ref{sub:algo:bagging} for details on how label switching is addressed. Please note that bagging can be used for any clustering algorithms. In Section~\ref{sub:sec:num} and Appendix~\ref{sub:sec:more-numerical} we report the results, both with and without bagging, for all sampling based algorithms. 
	
	\begin{algorithm}
		\caption{Bagging of Clustering Labels}\label{sub:algo:bagging}
		\SetKwData{Left}{left}\SetKwData{This}{this}\SetKwData{Up}{up}
		\SetKwFunction{Union}{Union}\SetKwFunction{FindCompress}{FindCompress}
		\SetKwInOut{Input}{input}\SetKwInOut{Output}{output}
		\Input{The label vectors $\lbrace \mathbf{l}_j \rbrace_{j=1}^b \in \mathbb{R}^N$ from $b$ independent runs. The number of clusters $K$. Note that each entry of $\mathbf{l}_j$ is a positive integer from $1$ to $K$. }
		\Output{The final label vector $\mathbf{l}_0 \in \mathbb{R}^N$.}
		
		\For  {$m=1$ \KwTo $b$}   {
			Write $\mathcal{M}_j = \lbrace r : \mathbf{l}_m(r) =  j\rbrace$, $j=1,...,K$. \\
			\For  {$i=1$ \KwTo $b$ \textbf{and} $i \neq m$}   {
				1. Write $\mathcal{I}_q =  \lbrace r : \mathbf{l}_i(r) =  q\rbrace$, $q=1,...,K$.
				Let $\mathbf{S} \in \mathbb{R}^{K \times K}$ be a score matrix where
				\[
				[\mathbf{S}]_{jq} = \frac{\mathbf{card}(\mathcal{M}_j \cap \mathcal{I}_q)}{\min(\mathbf{card}(\mathcal{M}_j),\ \mathbf{card}(\mathcal{I}_q))},\ 1\leq j,q \leq K.
				\]
				2. Switch the labels in $\mathbf{l}_i$ based on score matrix $\mathbf{S}$:\\
				\For  {$k=1$ \KwTo $K$}   {
					Let $q = \argmax_{j} [\mathbf{S}]_{jk}$. For $\forall r \in \mathcal{I}_q$ set $\mathbf{l}_i(r) := k$ .
				}
			}
		}
		\For  {$n=1$ \KwTo $N$}{
			Set $\mathbf{l}_0(n) := \mathbf{mode}(\lbrace \mathbf{l}_j(n) \rbrace_{j=1}^b)$.
		}
	\end{algorithm}
	The Step $2$ of Algorithm~\ref{sub:algo:bagging} has a subtle issue: there might exists an integer $q$ such that $q = \argmax_{j} [\mathbf{S}]_{jk} = \argmax_{j} [\mathbf{S}]_{jr}$. Please see our supplementary codes on how to tackle this problem.
	
	\section{Clustering Accuracy and Normalized Mutual Information}\label{sub:sec:metric}
	The clustering accuracy measures the percentage of correctly labeled data points \citep{you2016scalable}. It is calculated by 
	\begin{equation*}
	accr = \max_{\pi} \frac{100}{N} \sum_{i,j} l_{\pi(i)j}^{est} l_{ij}^{true}, \ 1\leq i \leq K,\ 1\leq j \leq N.
	\end{equation*}
	Here $\pi$ is the permutation of $K$ labels. The estimated label indicator $l_{\pi(i)j}^{est}$ equals to $1$ if and only if we assign label $\pi(i)$ to the $j$-th point, and $0$ otherwise. The ground-truth label indicator $l_{ij}^{true}$ equals to $1$ if and only if the $j$-th point has label $i$, and $0$ otherwise.
	
	The normalized mutual information \citep{strehl2002cluster} is calculated by
	\begin{equation*}
	\mbox{NMI}(\mathbf{l}^{est}, \mathbf{l}^{true}) = \frac{I(\mathbf{l}^{est}, \mathbf{l}^{true})}{\sqrt{H(\mathbf{l}^{est})H(\mathbf{l}^{true})}}.
	\end{equation*}
	Here $\mathbf{l}^{est}$ and $\mathbf{l}^{true}$ are estimated/ground-truth label vectors, respectively. We use $I(\mathbf{l}^{est}, \mathbf{l}^{true})$ to denote the mutual information between $\mathbf{l}^{est}$ and $\mathbf{l}^{true}$, and $H(\mathbf{l}^{est})$ to denote the entropy of $\mathbf{l}^{est}$. Similarly for $H(\mathbf{l}^{true})$.
	
	\section{Proofs of Main Theorems}\label{sub:sec:prorof-thms}
	In this section, we will prove the theorems from
	Section~\ref{sec:theory}. The following Lemmas are used to prove Theorem~\ref{sub:thm:sub-preserve-noisy}.
	\noindent
	\begin{lem}\label{lem:1}
		Let $\mathbf{b}$ be a vector sampled uniformly from $\mathbb{S}^{d-1}$, and $\lambda_k$ $(k=1,..,d)$ be constants such that $1 \geq \lambda_1 \geq \lambda_2 \geq ... \geq \lambda_d \geq 0$. For constant $g_1 \in (\lambda_d,\lambda_1)$, we write $r_i=(g_1^2-\lambda_i^2)_+$ and $s_i=(g_1^2-\lambda_i^2)_-$. Assuming that $\sum_{i=1}^d r_i > \sum_{i=1}^d s_i$, then 
		\[
		\mathbb{P} \left[ \sum_{i=1}^d \left( \lambda_{i} b_{i} \right) ^2 < g_1^2 \right] \geq 1-2 e^{-\epsilon^2},
		\]
		where 
		\[
		\epsilon = \frac{ \sum_{i=1}^d \left( r_i-s_i \right) }{\left( \sqrt{\sum_{i=1}^d r_i^2}+\sqrt{\sum_{i=1}^d s_i^2} \right) + \sqrt{\left( \sqrt{\sum_{i=1}^d r_i^2}+\sqrt{\sum_{i=1}^d s_i^2} \right)^2+2 s_1 \sum_{i=1}^d \left( r_i-s_i \right)}}.
		\]
	\end{lem}
	\begin{proof}
		We write $b_i=\frac{z_i}{\sqrt{\sum_{j=1}^d z_j^2}}$, where $\lbrace z_i \rbrace_{i=1}^d$ are i.i.d. $\mathcal{N}(0,1)$ random variables. The goal is to bound 
		\begin{align*}
		\mathbb{P} \left[\sum_{i=1}^d \left( g_1^2-\lambda_i^2 \right)_- \cdot z_i^2 \geq \sum_{i=1}^d \left( g_1^2-\lambda_i^2 \right)_+ \cdot z_i^2 \right]=\mathbb{P} \left[\sum_{i=1}^d s_i \cdot z_i^2 \geq \sum_{i=1}^d r_i \cdot z_i^2 \right].
		\end{align*}
		Note that $g_1 \in (\lambda_d,\lambda_1)$, hence both $\sum_{i=1}^d r_i$ and $\sum_{i=1}^d s_i$ are strictly positive. 
		
		Now we write $X=\sum_{i=1}^d s_i \cdot z_i^2$ and $Y=\sum_{i=1}^d r_i \cdot z_i^2$. Applying Lemma $1$ in \cite{laurent2000adaptive} we have for positive constants $\epsilon_a$ and $\epsilon_b$ the following inequalities are true
		\begin{align*}
		& \mathbb{P} \left[X \geq \sum_{i=1}^d s_i+ 2\sqrt{\sum_{i=1}^d s_i^2 } \epsilon_a+ 2s_1 \epsilon_a^2 \right]\leq e^{-\epsilon_a^2}, \quad \mathbb{P} \left[Y \leq \sum_{i=1}^d r_i - 2 \sqrt{\sum_{i=1}^d r_i^2 }\epsilon_b \right]\leq e^{-\epsilon_b^2}.
		\end{align*}
		
		We set $\epsilon_a=\epsilon_b$ and 
		\begin{align*}
		\sum_{i=1}^d s_i+ 2\sqrt{\sum_{i=1}^d s_i^2 } \epsilon_a+ 2s_1 \epsilon_a^2 =\sum_{i=1}^d r_i - 2 \sqrt{\sum_{i=1}^d r_i^2 }\epsilon_b.
		\end{align*}
		Solving the above quadratic equation we have 
		\begin{align*}
		\epsilon_a = \epsilon_b = \frac{ \sum_{i=1}^d \left( r_i-s_i \right)}{ \left( \sqrt{\sum_{i=1}^d r_i^2}+\sqrt{\sum_{i=1}^d s_i^2} \right) + \sqrt{\left( \sqrt{\sum_{i=1}^d r_i^2}+\sqrt{\sum_{i=1}^d s_i^2} \right) ^2+2 s_1 \sum_{i=1}^d \left( r_i-s_i \right)}}.
		\end{align*}
		Consequently
		\begin{align*}
		\mathbb{P} \left[ X \geq Y \right] & \leq  \mathbb{P}\left[ X \geq \sum_{i=1}^d s_i+ 2\sqrt{\sum_{i=1}^d s_i^2 } \epsilon_a+ 2s_1 \epsilon_a^2 \right]+\mathbb{P} \left[ Y \leq \sum_{i=1}^d r_i - 2 \sqrt{\sum_{i=1}^d r_i^2 }\epsilon_b \right]\\
		& \leq e^{-\epsilon_a^2}+e^{-\epsilon_b^2}.
		\end{align*}
		Substituting $\epsilon_a$ and $\epsilon_b$ into the inequality above yields the result.
	\end{proof}
	
	The following bound on F-distributed random variables follows from Lemma~\ref{lem:1}.
	
	\begin{coro}\label{coro:f_upper_bound}
		Let $X \sim F(m,n)$, and $m,n \geq 2$. Then for constant $q >1$, we have
		\[
		\mathbb{P} \left[X \geq q \right]\leq 2 e^{-\epsilon^2},
		\]
		where $\epsilon=\frac{1}{2} \left[- \left( \sqrt{m}+\frac{qm}{\sqrt{n}} \right) + \sqrt{\left( \sqrt{m}+\frac{qm}{\sqrt{n}} \right)^2+2m \left( q-1 \right)} \right]$.
	\end{coro}
	\begin{proof}
		We write $b_i =  \frac{z_i}{\sum_{i=1}^{m+n} z_i^2}$, and $X = \frac{(\sum_{i=1}^m z_i^2)/m}{(\sum_{i=m+1}^{m+n} z_i^2)/n}$, where $\lbrace z_i \rbrace_{i=1}^{m+n}$ are i.i.d. $\mathcal{N}(0,1)$ random variables. Then we have 
		\begin{align*}
		\mathbb{P}\left[X \geq q \right] = \mathbb{P}\left[\sum_{i=1}^m \frac{1}{mq} \cdot z_i^2 \geq \sum_{i=m+1}^{m+n} \frac{1}{n} \cdot z_i^2 \right] .
		\end{align*}
		The corollary follows by selecting $\lambda_i = \sqrt{\frac{1}{2}+\frac{1}{mq}}$ for $i=1,...,m$, $\lambda_i = \sqrt{\frac{1}{2} -\frac{1}{n}}$ for $i=m+1,...,m+n$, and $g_1 = \sqrt{\frac{1}{2}}$ in Lemma~\ref{lem:1}. 
	\end{proof}
	
	Lemma~\ref{sub:lem:beta-bound} states a bound on the order statistics of Beta distributed random variables.
	\begin{lem}\label{sub:lem:beta-bound}
		Suppose $T_l$ satisfy Assumption $A1$. For any $k=1,...,K$, let $\lbrace B_{(i)} \rbrace_{i=1}^{N_k-1}$ be the order statistics from a sample of $(N_k-1)$ i.i.d $\beta(\frac{1}{2},\frac{d-1}{2})$ random variables, then
		\[
		\mathbb{P} \left[ B_{(N_k - 1)} \geq \frac{1}{2} \right] \leq  2(N_k - 1)  e^{-{\epsilon_2}^2},
		\]
		and 
		\[
		\mathbb{P}\left[B_{(N_k - d_{\max})} \leq T_l^2 \right] \leq \frac{(N_k-d_{\max})}{d_{\max}\left( N_k+1 \right)  \left( N_k^{\rho}-1 \right)^2}.
		\]
	\end{lem}
	Here $\epsilon_2$ is defined in \eqref{sub:eq:beta_upper}.
	\begin{proof}
		Let $B \sim \beta(\frac{1}{2},\frac{1}{d-1})$. Then we can write $B = \frac{z_1^2}{\sum_{i=1}^d z_i^2}$, where $\lbrace z_i \rbrace_{i=1}^{d}$ are i.i.d. $\mathcal{N}(0,1)$ random variables. Select $\lambda_1 = 1$, $\lambda_i = 0$ ($i=2,..,d$) and and $g_1 = \frac{1}{\sqrt{2}}$ in Lemma~\ref{lem:1}. Note the following fact
		\begin{align*}
		\epsilon_2 = \frac{\sqrt{d-1}-1}{2 + \frac{1}{\sqrt{d-1} + 1}} \leq \frac{d - 2}{\left( \sqrt{d-1} + 1 \right) + \sqrt{\left( \sqrt{d-1} + 1 \right) + 2(d-2)}}.
		\end{align*}
		From Lemma~\ref{lem:1} we have $\mathbb{P} \left[ B \geq \frac{1}{2} \right] \leq 2e^{-{\epsilon_2}^2}$. Therefore by union bound inequality we have
		\begin{align*}
		\mathbb{P} \left[ B_{(N_k - 1)} \geq \frac{1}{2} \right] \leq  2(N_k - 1) e^{-{\epsilon_2}^2}.
		\end{align*}
		This proves the first part of Lemma~\ref{sub:lem:beta-bound}.
		
		Next we prove the second part of Lemma~\ref{sub:lem:beta-bound}. Let $U_{(i)}=F_{(\frac{1}{2},\frac{d-1}{2})}(B_{(i)})$, here $F_{(\frac{1}{2},\frac{d-1}{2})}$ is the CDF of the Beta distribution $\beta(\frac{1}{2},\frac{d-1}{2})$. Note that $\lbrace U_{(i)} \rbrace_{i=1}^{N_k-1}$ are the order statistics of the uniform distribution. 
		
		From Assumption $A1$ we know $F_{(\frac{1}{2},\frac{d-1}{2})}(T_l^2)\leq 1-\frac{d_{\max}}{N_k^{1-\rho}}$ and hence
		\begin{equation}\label{eq:lem3-eq1}
		\mathbb{P}\left[B_{(N_k-d_{\max})}\leq T_l^2 \right] \leq \mathbb{P}\left[U_{(N_k-d_{\max})} \leq  1-\frac{d_{\max}}{N_k^{1-\rho}}\right].
		\end{equation}
		By Chebyshev's inequality and basic properties of the uniform order statistics we have 
		\begin{align}\label{eq:lem3-eq2}
		\mathbb{P}\left[U_{(N_k-d_{\max})} \leq  1-\frac{d_{\max}}{N_k^{1-\rho}}\right] 
		\leq \frac{Var \left[U_{(N_k-d_{\max})} \right]}{\left( \frac{d_{\max}}{N_k}-\frac{d_{\max}}{N_k^{1-\rho}} \right)^2}=\frac{(N_k-d_{\max})}{d_{\max}(N_k+1)(N_k^{\rho}-1)^2}.
		\end{align}
		Combine \eqref{eq:lem3-eq1} and \eqref{eq:lem3-eq2} we know 
		\[
		\mathbb{P}\left[B_{(N_k-d_{\max})} \leq T_l^2 \right] \leq \frac{(N_k-d_{\max})}{d_{\max}(N_k+1)(N_k^{\rho}-1)^2}.
		\]
		This completes the proof. 
	\end{proof}
	Lemma~\ref{sub:lem:decompose-uniform} to Lemma~\ref{sub:lem:informative} are used to prove Theorem~\ref{sub:thm:cluster-dis-noisy}.
	
	\begin{lem}\label{sub:lem:decompose-uniform}
		Let $\mathbf{v}$ be a random vector that uniformly distributed on $\mathbb{S}^{d-1}$. Then we can decompose $\mathbf{v}$ into $\mathbf{v} = \left[ \sqrt{g} s, \sqrt{1-g} \mathbf{u} \right]$, where $g \sim \beta(\frac{1}{2},\frac{d-1}{2})$, $\mathbf{u} \sim U(\mathbb{S}^{d-2})$ and $\mathbb{P}\left[ s = 1 \right ]  = \mathbb{P}\left[ s = -1 \right ] = 0.5$ are three independent random variables. 
	\end{lem}
	\begin{proof}
		It is straightforward to see $\langle \mathbf{v},\mathbf{v} \rangle = [v_1^2,...,v_d^2]$ follows the Dirichlet distribution with parameters $\boldsymbol \alpha = (\frac{1}{2},...,\frac{1}{2}) \in \mathbb{R}^d$.  We can decompose $\langle \mathbf{v},\mathbf{v} \rangle$ into the following concatenation of two random components
		\[
		\left[ v_1^2,...,v_d^2 \right] = \left[ v_1^2, (1-v_1^2)\frac{\langle \mathbf{v}_{-1}, \mathbf{v}_{-1} \rangle}{1-v_1^2} \right].
		\]
		Since Dirichlet distribution is completely neutral \citep{lin2016dirichlet}, we know that $v_1^2$ is independent of $\frac{\langle \mathbf{v}_{-1}, \mathbf{v}_{-1} \rangle}{1-v_1^2}$, where $v_1^2 \sim \beta(\frac{1}{2}, \frac{d-1}{2})$ and $\frac{\langle \mathbf{v}_{-1}, \mathbf{v}_{-1} \rangle}{1-v_1^2} \sim Dir(\boldsymbol \alpha_{-1})$. From symmetry, we can set $\sqrt{g}s := v_1 $ and $\mathbf{u} := \frac{\mathbf{v}_{-1}}{\sqrt{1-v_1^2}}$, where the distributions of $g$, $\mathbf{u}$ and $s$ are specified in the statement of Lemma~\ref{sub:lem:decompose-uniform}. This completes the proof. 

	\end{proof}
	Let $\lbrace \mathbf{a}_i \rbrace_{i=1}^{N_k-1}$ be $(N_k - 1)$ vectors that are uniformly sampled from $\mathbb{S}^{d-1}$. From Lemma~\ref{sub:lem:decompose-uniform}, we know that for any $i = 1,...,N_k -1$, the value of $a_{i1}$ is independent of $\frac{[a_{i2},..,a_{id}]}{\sqrt{1-a_{i1}^2}}$. The following corollary is then a direct result from this fact. 
	\begin{coro}\label{coro:lem:decompose-uniform}
		Let $\lbrace \mathbf{a}_{(i)} \rbrace_{i=1}^{N_k-1}$ be a permutation of $\lbrace \mathbf{a}_i \rbrace_{i=1}^{N_k-1}$ sorted in ascending order of the absolute value of the first coordinate. 
		Then we can write
		\[
		\mathbf{a}_{(i)} = \left[ {a}_{(i)1}, \sqrt{1-{a}_{(i)1}^2} \mathbf{b}_{N_k-i} \right],
		\]
		where $\lbrace  \mathbf{b}_i \rbrace_{i=1}^{N_k-1}$ are i.i.d. uniform samples on $\mathbb{S}^{d-2}$.
	\end{coro}
	
	\begin{lem}\label{sub:lem:permeance}
		\citep[Lemma B.3]{lerman2012robust} Let  $\lbrace  \mathbf{b}_i \rbrace_{i=1}^{d_{max}}$ be i.i.d.~uniform samples from $\mathbb{S}^{d-2}$, $d \geq 3$. Then for any $t \geq 0$
		\begin{align*}
		\inf_{\| \mathbf{u} \|_2=1} \sum_{i=1}^{d_{max}} \left| \langle \mathbf{u}, \mathbf{b}_i \rangle \right| \geq \sqrt{\frac{2}{\pi}} \frac{d_{max}}{\sqrt{d -1}} - 2\sqrt{d_{max}} - t \sqrt{\frac{d_{max}}{d-2}},
		\end{align*}
		with probability at least $1-e^{-t^2/2}$.
	\end{lem}
	
	\begin{coro}\label{sub:coro:permeance}
		Use the same definition of $\lbrace \mathbf{b}_i \rbrace_{i=1}^{d_{max}}$ from Lemma~\ref{sub:lem:permeance}. Then for any $t\geq 0$:
		\begin{equation*}
		\sup_{\norm{\mathbf{u}}_2=1} \sum_{i=1}^{d_{max}} \langle \mathbf{u}, \mathbf{b}_i \rangle  \leq 2\sqrt{d_{max}} + t \sqrt{\frac{d_{max}}{d-2}},
		\end{equation*}
		with probability at least $1-e^{-t^2/2}$.
	\end{coro}
	\begin{proof}
		Note that $\mathbb{E} \left[\langle \mathbf{u}, \mathbf{b}  \rangle \right] = 0$ for any $\mathbf{b} \sim U(\mathbb{S}^{d-2})$ and $\mathbf{u} \in \mathbb{R}^{d-1}$. Therefore by Lemma 6.3 in \cite{ledoux2013probability} we have:
		\[
		\mathbb{E} \left[ \sup_{\norm{\mathbf{u}}_2=1} \sum_{i=1}^{d_{max}} \langle \mathbf{u}, \mathbf{b}_i \rangle \right] \leq 2 \sup_{\norm{\mathbf{u}}_2=1} \left[ \mathbb{E}\norm{ \sum_{i=1}^{d_{max}} \epsilon_i \mathbf{b}_i}^2 \right] = 2\sqrt{d_{max}}.
		\]
		Here $\lbrace \epsilon_i \rbrace_{i=1}^{d_{max}}$ are i.i.d. Rademacher random variables. The lemma is proved by following similar steps after equation (B.11) in \cite{lerman2012robust}. 
	\end{proof}
	
	\begin{lem}\label{sub:lem:informative}
		Suppose Assumption $A3$.
		Write $\mathbf{a}_0 = \left[1,0,...,0 \right] \in \mathbb{R}^{d}$, and use the definitions for $\lbrace \mathbf{a}_i \rbrace_{i=1}^{N_k-1}$ and $\lbrace \mathbf{a}_{(i)} \rbrace_{i=1}^{N_k-1}$ from Corollary~\ref{coro:lem:decompose-uniform}. Let $\mathbf{B} \in \mathbb{R}^{d \times (d_{max} + 1)}$ be a matrix where its first column is $\mathbf{a}_0$ and its $i$-th column ($2 \leq i \leq d_{max} + 1$) is $\mathbf{a}_{(N_k - i + 1)}$. Let the largest $d$  singular values of $\mathbf{B}$ be $s_1\geq s_2\geq\cdots\geq s_d$.  Then we have
		\begin{equation*}
		\mathbb{P}\left[ s_d^2 \geq q_0 \right ] \geq 1-\frac{2}{{N}^{t^2/2}} -  \frac{(N_k-d_{\max})}{d_{\max}(N_k+1)(N_k^{\rho}-1)^2} - 2(N_k - 1)  e^{-{\epsilon_2}^2},
		\end{equation*}
		where $\epsilon_2$ is defined in \eqref{sub:eq:beta_upper}.
	\end{lem}
	\begin{proof}
		From Corollary~\ref{coro:lem:decompose-uniform}, we know $\mathbf{B}$ can be re-written as 
		\begin{equation*}
		\mathbf{B} = \begin{pmatrix}
		1,& {a}_{(N_k-1)1},& ... & {a}_{(N_k-d_{max})1}\\ 
		\mathbf{0},& \sqrt{1 - {a}_{(N_k-1)1}^2}\mathbf{b}_{1}, & ... & \sqrt{1 - {a}_{(N_k-d_{max})1}^2} \mathbf{b}_{d_{max}}
		\end{pmatrix},
		\end{equation*}
		where $\lbrace  \mathbf{b}_i \rbrace_{i=1}^{d_{max}}$ are i.i.d. uniform samples from $\mathbb{S}^{d-2}$. 
		
		Given the dimensions of $\mathbf{B}$, we know $s_d = \inf_{\| \mathbf{x} \|_2=1}  \norm{ \mathbf{B}^T \mathbf{x} }_2$. For convenience, we write \[
		\mathbf{x}' = \frac{1}{\sqrt{1-x_1^2}} \left[ x_2,...,x_d \right],
		\]
		where $\| \mathbf{x}' \|_2 = 1$, $c_i = \langle \mathbf{x}', \mathbf{b}_i \rangle $,  ${a}_{(N_k)1} = 1$. Let $\mathcal{E}_1$ be the event that $\lbrace s_d^2 \geq q_0 \rbrace$, and $\mathcal{E}_2$ be the event that $\lbrace \textit{${a}_{(N_k-i)1}^2 \in \left[T_l^2, \frac{1}{2} \right]$, $\forall i = 1,..,d_{max}$} \rbrace$. From Lemma~\ref{sub:lem:beta-bound} we know
		\begin{equation}\label{sub:eq:lem6-beta-bound}
		\mathbb{P} \left[ \mathcal{E}_2 \right] \geq 1 -  \frac{(N_k-d_{\max})}{d_{\max}(N_k+1)(N_k^{\rho}-1)^2} - 2(N_k - 1)  e^{-{\epsilon_2}^2}.
		\end{equation}
		Conditioning on $\mathcal{E}_2$, we have the following relations
		\begin{align}\label{sub:eq:lem6}
		\norm{ \mathbf{B}^T \mathbf{x} }_2^2= & \left \Vert \begin{pmatrix}
		1,& {a}_{(N_k-1)1},& ... & {a}_{(N_k-d_{max})1}\\ 
		\mathbf{0},& \sqrt{1 - {a}_{(N_k-1)1}^2}\mathbf{b}_{1}, & ... & \sqrt{1 - {a}_{(N_k-d_{max})1}^2} \mathbf{b}_{d_{max}}
		\end{pmatrix}^T \mathbf{x} \right \Vert_2^2 \nonumber \\
		= & \sum_{i=0}^{d_{max}} \left({a}_{(N_k-i)1}x_1 + \sqrt{\left(1-{a}_{(N-i)1}^2 \right) \left( 1-x_1^2 \right)} c_i \right)^2  \nonumber \\
		= & \sum_{i=0}^{d_{max}} {a}_{(N-i)1}^2 x_1^2 + 2\sum_{i=0}^{d_{max}} \sqrt{{a}_{(N-i)1}^2 (1-{a}_{(N_k-i)1}^2 )(1-x_1^2)} c_i x_1 \nonumber \\
		& + \sum_{i=1}^{d_{max}} (1-{a}_{(N_k-i)1}^2 )(1-x_1^2)c_i^2 \nonumber \\
		\geq & T_l^2d_{max} \cdot x_1^2 - \sqrt{(1-x_1^2)x_1^2} \sup_{\norm{u}_2=1} \sum_{i=1}^{d_{max}} \langle \mathbf{u},\mathbf{b}_i \rangle \nonumber \\
		& + \frac{1-x_1^2}{2} \inf_{\norm{u}_2=1} \sum_{i=1}^{d_{max}} \langle \mathbf{u},\mathbf{b}_i \rangle^2.
		\end{align}
		From Lemma~\ref{sub:lem:permeance} and Corollary~\ref{sub:coro:permeance} and conditional on $\mathcal{E}_2$, we have the following inequality 
		\begin{equation}\label{sub:eq:lem6-2}
		\eqref{sub:eq:lem6} \geq (1-x_1^2)C_2 - \sqrt{(1-x_1^2)x_1^2}C_1 +T_l^2d_{max}\cdot x_1^2,
		\end{equation}
		with probability at least $ 1-\frac{2}{{N}^{t^2/2}}$. Since $1-x_1^2 \leq 1$, a lower bound of the RHS of \eqref{sub:eq:lem6-2} is
		\begin{align*}
		\left( T_l^2d_{max}-C_2 \right) x_1^2 - C_1x_1 + C_2 \geq \frac{\left( T_l^2d_{max} - C_2 \right)C_2 - \frac{C_1^2}{4}}{T_l^2d_{max}} \geq q_0, \nonumber
		\end{align*} 
		where the $q_0$ comes from Assumption A3. Finally, note the following fact
		\begin{align}
		\mathbb{P}\left[\mathcal{E}_1 \right] &\geq \mathbb{P}\left[\mathcal{E}_1 \middle|  \mathcal{E}_2\right]+\mathbb{P}\left[ \mathcal{E}_2 \right]-1  \\
		& = 1-\frac{2}{{N}^{t^2/2}} -  \frac{(N_k-d_{\max})}{d_{\max}(N_k+1)(N_k^{\rho}-1)^2} - 2(N_k - 1)  e^{-{\epsilon_2}^2}. \nonumber
		\end{align}
		This completes the proof. 
		
	\end{proof}
	
	
	\noindent
	\begin{proof}\textbf{of Theorem~\ref{sub:thm:sub-preserve-noisy}} Let the event $\mathcal{E}_{1i} = \lbrace \mathbf{Y}_{\mathcal{C}_i} \mbox{ only contains points in same subspace} \rbrace$, then $\mathcal{E}_1 = \cap_{i=1}^n \mathcal{E}_{1i}$ is the event that Algorithm~\ref{sub:algo:sbsc} has sub-cluster preserving property. Let the event $\mathcal{E}_2 = \lbrace \sigma \norm{\mathbf{e}_i^{(k)}}_2<g_2, \forall i,k \rbrace $, where $g_2$ is from Assumption $A2$.  
		Our goal is to find a lower bound on $\mathbb{P}\left[ \mathcal{E}_1 \right]$. 
		
		Note the following fact
		\begin{equation}
		\mathbb{P}\left[\mathcal{E}_1\right] \geq 1-\sum_{i=1}^n \mathbb{P}\left[ \mathcal{E}^\complement_{1i} |  \mathcal{E}_2 \right] + \mathbb{P} \left[ \mathcal{E}_2 \right]-1 = \mathbb{P} \left[ \mathcal{E}_2 \right]-\sum_{i=1}^n \mathbb{P} \left[ {\mathcal{E}}^\complement_{1i} |  \mathcal{E}_2 \right]. 
		\end{equation}
		Therefore, it suffices to find a lower bound on $\mathbb{P} \left[ \mathcal{E}_2 \right]-\sum_{i=1}^n \mathbb{P} \left[ {\mathcal{E}}^\complement_{1i} |  \mathcal{E}_2 \right]$. 
		
		We start by finding a preliminary upper bound on $\mathbb{P}\left[{\mathcal{E}}^\complement_{11} | \mathcal{E}_2 \right]$.
		WLOG assume that $\mathbf{y_1}^{(1)}$ is one of the sampled points, and $\mathbf{Y}_{\mathcal{C}_1}$ is the sub-cluster associated with it. Recall that in Step 2 of Algorithm~\ref{sub:algo:sbsc}, we use $|\langle \mathbf{y}_1^{(1)}, \mathbf{y}_i^{(k)} \rangle|$ to measure the affinity between $\mathbf{y}_1^{(1)}$ and $\mathbf{y}_i^{(k)}$, the nearest $(d_{max} + 1)$ points are then used to construct the sub-cluster associated with $\mathbf{y}_1^{(1)}$. Write $\hat{A}^k = \lbrace |\langle \mathbf{y}_1^{(1)}, \mathbf{y}_i^{(k)} \rangle| \rbrace_{i=1}^{N_k}$, for $\mathcal{E}_{11}$ to happen we need the largest $(d_{\max}+1)$ values among  $\cup_{k=1}^K \hat{A}^k $ to be from the set  $\hat{A}^1$. Mathematically this means 
		\[{\mathcal{E}}^\complement_{11}=\left\{
		\hat{A}_{(N_1 - d_{\max})}^1 \leq \max_{k \neq 1} \max_{i=1,..,N_k} \hat{A}_{i}^k\right\},
		\]
		where $\hat{A}_{i}^k$ is the $i$-th element in $\hat{A}^k$ and $\hat{A}_{(i)}^k$ is the $i$-th smallest element in $\hat{A}^k$.
		
		Recall from \eqref{sub:eq:normalization} that $\mathbf{y_i}^{(k)}=\frac{\mathbf{U_k} \mathbf{a}_i^{(k)}+\sigma\mathbf{e}_i^{(k)}}{\norm{\mathbf{U}_k \mathbf{a}_i^{(k)}+\sigma\mathbf{e}_i^{(k)}}_2}$.  The triangle inequality tells us that
		\[
		\norm{\mathbf{U}_k \mathbf{a}_i^{(k)}}_2-\norm{\sigma \mathbf{e}_i^{(k)}}_2 \leq \norm{\mathbf{U}_k \mathbf{a}_{i}^{(k)}+ \sigma \mathbf{e}_i^{(k)}}_2 \leq \norm{\mathbf{U}_k \mathbf{a}_i^{(k)}}_2+ \norm{\sigma \mathbf{e}_i^{(k)}}_2.
		\]
		Therefore conditional on $\mathcal{E}_2$, we know the normalizing constants ${\norm{\mathbf{U}_k \mathbf{a}_i^{(k)}+\sigma\mathbf{e}_i^{(k)}}_2}$ are bounded in $[1-g_2,1+g_2]$. We can write $A_i^{k} = \norm{\mathbf{y}_{1}^{(1)}}_2 \cdot \norm{\mathbf{y}_{i}^{(k)}}_2 \cdot \hat{A}_i^k $. It is fairly straightforward to get the following relation
		\begin{align}
		\mathbb{P} \left[A_{(N_1 - d_{\max})}^1 \leq \frac{1+g_2}{1-g_2}\max_{k \neq 1}\max_{1\leq i \leq N_k} A_i^k \middle| \mathcal{E}_2 \right] \geq  \mathbb{P}\left[{\mathcal{E}}^\complement_{11} \middle| \mathcal{E}_2 \right]. \label{eq:thm1-pt1}
		\end{align}
		
		Conditioning on $\mathcal{E}_2$ and write $B_i = \langle \mathbf{a}_1^{(1)} , \mathbf{a}_i^{(1)} \rangle ^2$, $ i=2,...,N_1-1$. We have the following inequalities
		\begin{align*}
		A_{(N_1 - d_{\max})}^1 =&  \left|\sqrt{B_{(N_1-d_{\max})}}+ \sigma \langle  \mathbf{U}_1 \mathbf{a}_1^{(1)},\mathbf{e}_i^{(1)} \rangle + \sigma \langle \mathbf{U}_1 \mathbf{a}_i^{(1)},\mathbf{e}_1^{(1)} \rangle + \sigma^2 \langle \mathbf{e}_1^{(1)}, \mathbf{e}_i^{(1)} \rangle \right| \\
		\geq &  \sqrt{B_{(N_1-d_{\max})}} -\sigma \norm{\mathbf{e}_i^{(1)}}_2 - \sigma \norm{\mathbf{e}_1^{(1)}}_2 - \sigma^2 \norm{\mathbf{e}_1^{(1)}}_2 \max_{i \neq 1} \norm{\mathbf{e}_i^{(1)}}_2 \\
		\geq & \sqrt{B_{(N_1-d_{\max})}}-2g_2-g_2^2.
		\end{align*}
		Similarly we have
		\begin{align*}
		\max_{k \neq 1}\max_{1\leq i \leq N_k}  A_i^k   =  &\max_{k \neq 1}\max_{1\leq i \leq N_k} \left | \langle \mathbf{U}_1 \mathbf{a}_1^{(1)} , \mathbf{U}_k \mathbf{a}_i^{(k)} \rangle+ \sigma \langle  \mathbf{U}_1 \mathbf{a}_1^{(1)},\mathbf{e}_i^{(k)} \rangle + \sigma \langle \mathbf{U}_k \mathbf{a}_i^{(k)},\mathbf{e}_1^{(1)} \rangle +\sigma^2 \langle \mathbf{e}_1^{(1)}, \mathbf{e}_i^{(k)} \rangle \right| \nonumber \\
		\leq  & \max_{k \neq 1}\max_{1\leq i \leq N_k} \left |\langle \mathbf{U}_1 \mathbf{a}_1^{(1)} , \mathbf{U}_k \mathbf{a}_i^{(k)} \rangle \right|+ \sigma \max_{k \neq 1}\max_{1\leq i \leq N_k} \norm{\mathbf{e}_i^{(k)}}_2 \\
		& + \sigma \norm{\mathbf{e}_1^{(1)}}_2+ \sigma^2 \norm{\mathbf{e}_1^{(1)}}_2 \max_{k \neq 1} \max_{1\leq i \leq N_k} \norm{\mathbf{e}_i^{(k)}}_2 \\
		\leq  & \max_{k \neq 1}\max_{1\leq i \leq N_k} \left|\langle \mathbf{U}_1 \mathbf{a}_1^{(1)} , \mathbf{U}_k \mathbf{a}_i^{(k)} \rangle \right| + 2g_2 + g_2^2.
		\end{align*}
		Pick $T$ from Assumption $A2$, then the LHS of \eqref{eq:thm1-pt1} has the following upper bound
		\begin{align}\label{eq:thm-1-goal-unnorm}
		\mathbb{P} \left[ T \leq Q \middle| \mathcal{E}_2 \right] +\mathbb{P} \left[B_{(N_1-d_{\max})}\leq T^2 \middle| \mathcal{E}_2 \right], 
		\end{align}
		where
		\begin{align*}
		Q= \left( 1+\frac{1+g_2}{1-g_2} \right) \left( 2g_2 + g_2^2 \right) +\frac{1+g_2}{1-g_2} \max_{k \neq 1}\max_{1\leq i \leq N_k}\left|\langle \mathbf{U}_1 \mathbf{a}_1^{(1)} , \mathbf{U}_k \mathbf{a}_i^{(k)} \rangle \right|.
		\end{align*}
		Now we are going to complete our proof in $3$ steps.
		
		\noindent
		\textbf{Step 1:} For the first term in \eqref{eq:thm-1-goal-unnorm} we have
		\begin{align*} 
		\mathbb{P} \left[T \leq Q \middle| \mathcal{E}_2 \right] = \mathbb{P}\left[g_{1} \leq \max_{k \neq 1}\max_{1\leq i \leq N_k}  \left|\langle \mathbf{U}_1 \mathbf{a}_1^{(1)} , \mathbf{U}_k \mathbf{a}_i^{(k)} \rangle \right| \right]. 
		\end{align*}
		From singular value decomposition we can write
		\[
		\langle \mathbf{U}_1 \mathbf{a}_1^{(1)} , \mathbf{U}_k \mathbf{a}_i^{(k)} \rangle= \mathbf{a}_1^{(1)T} \mathbf{W}_{1k} \mathbf{\Lambda}_{1k} \mathbf{V}_{1k}^T \mathbf{a}_i^{(k)}:=\mathbf{b}_k^T \mathbf{\Lambda}_{1k} \mathbf{V}_{1k}^T \mathbf{a_i}^{(k)},
		\]
		where both $\lbrace \mathbf{b}_k \rbrace_{k=2}^K$ and $\lbrace \mathbf{V}_{1k}^T \mathbf{a}_i^{(k)} \rbrace_{k=2}^K$ are sampled uniformly from $\mathbb{S}^{d-1}$. Therefore
		\begin{align}
		\mathbb{P} \left[g_{1} \leq \max_{k \neq 1}\max_{1\leq i \leq N_k} \left|\langle \mathbf{U}_1 \mathbf{a}_1^{(1)} , \mathbf{U}_k \mathbf{a}_i^{(k)} \rangle \right| \right] \nonumber 
		=&\mathbb{P} \left[ g_{1}^2 \leq  \max_{k \neq 1} \max_{1\leq i \leq N_k} \left( \mathbf{b}_k^T \mathbf{\Lambda}_{1k}\mathbf{V}_{1k}^T \mathbf{a}_i^{(k)} \right)^2 \right] \nonumber \\
		\leq & \sum_{k=2}^K \mathbb{P}\left[g_{1}^2 \leq \max_{1\leq i \leq N_k} \left(\mathbf{b}_k^T \mathbf{\Lambda}_{1k} \mathbf{V}_{1k}^T \mathbf{a}_i^{(k)}\right)^2 \right] \label{eq:thm1-eq1} \\
		\leq & \sum_{k=2}^K \mathbb{P}\left[g_{1}^2 \leq \sum_{i=1}^d \left(\lambda_{i}^{(1k)} b_{ki}\right)^2 \right] \label{eq:thm1-eq2} \\
		\leq & \sum_{k=2}^K \mathbb{P}\left[g_{1}^2 \leq \sum_{i=1}^d \left( \lambda_{i}^{(1)} b_{ki} \right)^2 \right], \label{eq:thm1-eq3}
		\end{align}
		where inequality~\eqref{eq:thm1-eq1} uses the union bound inequality,  \eqref{eq:thm1-eq2} comes from Cauchy-Schwarz inequality, and \eqref{eq:thm1-eq3} uses Definition~\ref{defn:max-affinity}. Since $\lbrace \mathbf{b}_k \rbrace_{k=2}^K \sim U(\mathbb{S}^{d-1})$, we can write \eqref{eq:thm1-eq3} as
		\[
		(K-1)\mathbb{P} \left[ g_{1}^2 \leq \sum_{i=1}^d \left( \lambda_{i}^{(1)} b_{i} \right)^2 \right],
		\]
		where $\mathbf{b}$ is uniformly distributed on $\mathbb{S}^{d-1}$. Now we apply Lemma~\ref{lem:1} directly to the quantity above and get $\mathbb{P}\left[g_{1}^2 \leq \sum_{i=1}^d \left(\lambda_{i}^{(1)} b_{i}\right)^2\right] \leq 2 e^{-\epsilon'^2}$ where
		\begin{align}
		\epsilon' & =\frac{ \sum_{i=1}^d(r_i-s_i)}{(\sqrt{\sum_{i=1}^d r_i^2}+\sqrt{\sum_{i=1}^d s_i^2})+\sqrt{(\sqrt{\sum_{i=1}^d r_i^2}+\sqrt{\sum_{i=1}^d s_i^2})^2+2 s_1 \sum_{i=1}^d(r_i-s_i)}} \nonumber \\
		& = \frac{-(\sqrt{\sum_{i=1}^d r_i^2}+\sqrt{\sum_{i=1}^d s_i^2})+\sqrt{(\sqrt{\sum_{i=1}^d r_i^2}+\sqrt{\sum_{i=1}^d s_i^2})^2+2 s_1 \sum_{i=1}^d(r_i-s_i)}}{2s_1} \nonumber  \\
		& \geq \frac{-(\sqrt{\sum_{i=1}^d r_i^2}+\sqrt{\sum_{i=1}^d s_i^2})+\sqrt{(\sqrt{\sum_{i=1}^d r_i^2}+\sqrt{\sum_{i=1}^d s_i^2})^2+2 \sum_{i=1}^d(r_i-s_i)}}{2} \label{eq:thm1-eq4}\\
		& = \frac{\sum_{i=1}^d(r_i-s_i)}{(\sqrt{\sum_{i=1}^d r_i^2}+\sqrt{\sum_{i=1}^d s_i^2})+\sqrt{(\sqrt{\sum_{i=1}^d r_i^2}+\sqrt{\sum_{i=1}^d s_i^2})^2+2 \sum_{i=1}^d(r_i-s_i)}} \nonumber \\
		& \geq \frac{\sum_{i=1}^d(r_i-s_i)}{2 \sqrt{\sum_{i=1}^d r_i^2}+\sqrt{4 \sum_{i=1}^d r_i^2+2 \sum_{i=1}^d r_i}} \geq  \epsilon_1 \nonumber.
		\end{align}
		Here $\epsilon_1$ is defined in \eqref{sub:eq:thm1-epsilon}, $r_i$ and $s_i$ are defined in Lemma~\ref{lem:1}, and \eqref{eq:thm1-eq4} comes from the following fact for positive constants $a$, $b$ and $s \in (0,1)$
		\begin{align*}
		& \frac{-a+\sqrt{a^2+2sb}}{2s} \geq \frac{-a+\sqrt{a^2+2b}}{2}.
		\end{align*}
		Therefore we have
		\[
		\mathbb{P} \left[g_{1} \leq \max_{k \neq 1}\max_{1\leq i \leq n_k} \left|\langle \mathbf{U}_1 \mathbf{a}_1^{(1)} , \mathbf{U}_k \mathbf{a}_i^{(k)} \rangle\right|\right] \leq 2(K-1) e^{-\epsilon_1^2}.
		\]
		\textbf{Step 2:} For the second term of \eqref{eq:thm-1-goal-unnorm}, we just need to use Lemma~\ref{sub:lem:beta-bound}. Note that for fixed $\mathbf{a}_1^{(1)}$, one can show that $B_i=\langle \mathbf{a}_1^{(1)},\mathbf{a}_i^{(1)} \rangle^2$ can be treated as a sample from a Beta distribution with parameters $(\frac{1}{2},\frac{d-1}{2})$. From Lemma~\ref{sub:lem:beta-bound} and Assumption $A2$ we have
		\begin{equation*}
		\mathbb{P} \left[B_{(N_1-d_{\max})} \leq T^2 \middle| \mathcal{E}_2 \right] \leq \frac{(N_1-d_{\max})}{d_{\max}(N_1+1)(N_1^{\rho}-1)^2}.
		\end{equation*}
		Combine the results above we know
		\begin{align}\label{sub:eq:thm1-beta-bound}
		\mathbb{P} \left[{\mathcal{E}}^\complement_{11} \middle| \mathcal{E}_2 \right] \leq 2(K-1) e^{-{\epsilon_1}^2} + \frac{(N_1-d_{\max})}{d_{\max}(N_1+1)(N_1^{\rho}-1)^2}.
		\end{align}
		\noindent
		\textbf{Step 3:} Now we are going to find the lower bound on $\mathbb{P}[\mathcal{E}_2]$. Let $\mathbf{e}$ be an independent copy of $\mathbf{e}_1^{(1)}$, note that $\norm{\mathbf{e}}_2^2/{D} \sim F_{D,d}$. From Corollary~\ref{coro:f_upper_bound} we have
		\begin{align*}
		\mathbb{P}\left[g_2 \leq \sigma||\mathbf{e}||_2 \right] & =\mathbb{P}\left[\frac{ g_2^2}{D \sigma^2} \leq \frac{||\mathbf{e}||_2^2}{D}\right] \leq 2 e^{-t^2},
		\end{align*}
		where $t$ can be calculated from Corollary~\ref{coro:f_upper_bound}. Using Assumption $A2$ we have 
		\begin{align*}
		t > \frac{D \left( \frac{g_2^2}{D\sigma^2}-1 \right) }{2 \left( \sqrt{D}+\frac{g_2^2}{\sigma^2 \sqrt{d}}+\sqrt{d} \right)} = \frac{\sqrt{d}}{2}\left( 1 - \frac{1 + \frac{d}{D} + \sqrt{\frac{d}{D}}}{1 + \frac{d}{D}+\frac{g_2^2}{D\sigma^2}} \right)\geq \left( 1 + \frac{\eta}{2+\eta} \right) \sqrt{\log{N}}.
		\end{align*}
		Therefore we have $\mathbb{P}\left[g_2 \leq \sigma ||\mathbf{e}||_2\right] \leq \frac{2}{N^{\left(1 + \frac{\eta}{2+\eta}\right)^2}}$. Now we note that
		\begin{align}
		\mathbb{P}\left[g_2 > \sigma\max_{k=1,...,K}\max_{1\leq i \leq N_k} \norm{\mathbf{e}_i^{(k)}}_2\right] &= \Pi_{i=1}^{N} \left(1-\mathbb{P}\left[g_2 \leq  \sigma \norm{\mathbf{e}_i^{(k)}}_2\right] \right) \nonumber \\
		& \geq  (1-2 e^{-t^2})^N \nonumber  \geq 1 - \frac{2N}{N^{\left(1 + \frac{\eta}{2+\eta}\right)^2}},  
		\end{align}
		where the last inequality comes from the Bernoulli's inequality. Therefore 
		\begin{align}\label{sub:eq:thm1-noise-bound}
		\mathbb{P}\left[\mathcal{E}_2\right]  \geq 1-\frac{2N}{N^{\left(1 + \frac{\eta}{2+\eta}\right)^2}}.
		\end{align}
		Finally, the above arguments hold for any $\mathbf{y}_i^{(k)}$. Putting \eqref{sub:eq:thm1-beta-bound} and \eqref{sub:eq:thm1-noise-bound} together and applying the union bound inequality yields the result
		\begin{align}\label{eq:thm1-final}
		\mathbb{P}[\mathcal{E}_1] \geq 1- \sum_{k=1}^K  \frac{n_k(N_k-d_{\max})}{d_{\max}(N_k+1)(N_k^{\rho}-1)^2}-2(K-1)n e^{-\epsilon_1^2}-\frac{2N}{N^{\left(1 + \frac{\eta}{2+\eta}\right)^2}}.
		\end{align}
	\end{proof}
	To prove Theorem~\ref{sub:thm:cluster-dis-noisy}, we will use the following equation
	\begin{equation}\label{sub:eq:ridge}
	(\mathbf{W}^T\mathbf{W}+\lambda \mathbf{I}_{d_2})^{-1} \mathbf{W}^T = \mathbf{W}^T (\mathbf{W} \mathbf{W}^T + \lambda \mathbf{I}_{d_1})^{-1},
	\end{equation}
	where $\mathbf{W} \in \mathbb{R}^{d_1 \times d_2}$ and $\lambda$ is a positive constant \citep[Chapter~4]{murphy2012machine}. Throughout the proof of Theorem~\ref{sub:thm:cluster-dis-noisy}, the subscript of identity matrix $\mathbf{I}$ will be omitted as its dimension is clear from the context. 
	\vspace{0.3cm}
	\begin{proof}\textbf{of Theorem~\ref{sub:thm:cluster-dis-noisy}}  Similar to the proof of Theorem~\ref{sub:thm:sub-preserve-noisy}, let $\mathcal{E}_1$ be the event that correct neighborhood property holds for all $\lbrace \mathbf{Y}_{\mathcal{C}_j} \rbrace_{j=1}^n$, let $\mathcal{E}_2$ be the event $\lbrace \sigma \| \mathbf{e}_i^{(k)} \|_2 < g, \forall i,k \rbrace$ ($g$ is from Assumption $A4$), $\mathcal{E}_3$ is the event that the smallest singular value of $\mathbf{B}\mathbf{B}^T$ is at least $q_0$, $\forall i=1,...,n$, and $\mathcal{E}_4$ is the event that the sub-cluster preserving property is satisfied. 
		
		Define $\mathcal{I} = \lbrace (i,j):$ the $i$-th and the $j$-th sampled points belong to different clusters, $1\leq i < j \leq n\rbrace$, and $\mathcal{J} = \lbrace (i,j):$ the $i$-th and the $j$-th sampled points belong to the same cluster, $1\leq i < j \leq n\rbrace$. Conditional on $\mathcal{E}_4$, we know that $\mathbf{Y}_{\mathcal{C}_i}$ and $\mathbf{Y}_{\mathcal{C}_j}$ belong to different clusters if $(i,j) \in \mathcal{I}$, and belong to the same cluster if $(i,j) \in \mathcal{J}$. 
		
		We will show that conditioning on $\mathcal{E}_2$, $\mathcal{E}_3$ and $\mathcal{E}_4$, there exists a constant $l$ such that 
		\begin{align*}
		\mathbb{P}\left[\mathcal{E}_1 \middle| \mathcal{E}_2, \mathcal{E}_3,\mathcal{E}_4 \right] \geq \mathbb{P}\left[d(\mathbf{Y}_{\mathcal{C}_i},\mathbf{Y}_{\mathcal{C}_j})_{\forall (i,j) \in \mathcal{I}} > l \middle| \mathcal{E}_2,\mathcal{E}_3,\mathcal{E}_4  \right] \geq 1 - \sum_{\forall (i,j) \in \mathcal{I}} \mathbb{P}\left[d(\mathbf{Y}_{\mathcal{C}_i},\mathbf{Y}_{\mathcal{C}_j}) \leq l \middle| \mathcal{E}_2,\mathcal{E}_3,\mathcal{E}_4  \right].
		\end{align*}
		Then we obtain an upper bound on $\mathbb{P}[d(\mathbf{Y}_{\mathcal{C}_i},\mathbf{Y}_{\mathcal{C}_k}) \leq l \mid \mathcal{E}_2,\mathcal{E}_3, \mathcal{E}_4]$, $\forall (i,k) \in \mathcal{I}$. The theorem will follow by using the union bound inequality.
		
		WLOG assume that $\mathbf{Y}_{\mathcal{C}_1}$ and $\mathbf{Y}_{\mathcal{C}_2}$ belong to $\mathcal{S}_1$, and $\mathbf{Y}_{\mathcal{C}_3}$ belongs to $\mathcal{S}_2$. 
		The distance function $d(\mathbf{Y}_{\mathcal{C}_1},\mathbf{Y}_{\mathcal{C}_2})$ can be explicitly written as 
		\begin{equation}\label{eq:thm2-eq1}
		\norm{\mathbf{Y}_{\mathcal{C}_1}-\mathbf{Y}_{\mathcal{C}_2}(\mathbf{Y}_{\mathcal{C}_2}^T \mathbf{Y}_{\mathcal{C}_2}+\lambda \mathbf{I})^{-1}\mathbf{Y}_{\mathcal{C}_2}^T\mathbf{Y}_{\mathcal{C}_1}}_F + \norm{\mathbf{Y}_{\mathcal{C}_2}-\mathbf{Y}_{\mathcal{C}_1}(\mathbf{Y}_{\mathcal{C}_1}^T \mathbf{Y}_{\mathcal{C}_1}+\lambda \mathbf{I})^{-1}\mathbf{Y}_{\mathcal{C}_1}^T\mathbf{Y}_{\mathcal{C}_2}}_F.
		\end{equation}
		Conditional on $\mathcal{E}_4$, we can write $\mathbf{Y}_{\mathcal{C}_1}=\mathbf{U}_1 \hat{\mathbf{B}}_1+\hat{\mathbf{E}}_1$, where $\|  [\mathbf{U}_1 \hat{\mathbf{B}}_{1}]_j + [\hat{\mathbf{E}}_{1}]_j \|_2 = 1$. Let $\mathbf{B}_1$ and $\mathbf{E}_1$ be the ``un-normalized'' version of $\hat{\mathbf{B}}_1$ and $\hat{\mathbf{E}}_1$ respectively. Here each column of $\mathbf{B}_1$ is a sample from the uniform distribution on $\mathbb{S}^{d-1}$. We have the following relation
		\begin{align*}
		[\hat{\mathbf{B}}_{1}]_j = \frac{[\mathbf{B}_{1}]_j}{\norm{[\mathbf{U}_1\mathbf{B}_{1}]_j + [\mathbf{E}_{1}]_j}_2}, \  [\hat{\mathbf{E}}_{1}]_j = \frac{[\mathbf{E}_{1}]_j}{\norm{[\mathbf{U}_1\mathbf{B}_{1}]_j + [\mathbf{E}_{1}]_j}_2}, \  j = 1,...,d_{max} + 1.
		\end{align*}
		Similarly we can write $\mathbf{Y}_{\mathcal{C}_2}=\mathbf{U}_1 \hat{\mathbf{B}}_2+\hat{\mathbf{E}}_2$ and $\mathbf{Y}_{\mathcal{C}_3}=\mathbf{U}_2 \hat{\mathbf{B}}_3+\hat{\mathbf{E}}_3$. Using equation \eqref{sub:eq:ridge}, the first term in \eqref{eq:thm2-eq1} can be rewritten as 
		\begin{align}
		&\norm{\mathbf{Y}_{\mathcal{C}_1}-\mathbf{Y}_{\mathcal{C}_2}(\mathbf{Y}_{\mathcal{C}_2}^T \mathbf{Y}_{\mathcal{C}_2}+\lambda \mathbf{I})^{-1}\mathbf{Y}_{\mathcal{C}_2}^T\mathbf{Y}_{\mathcal{C}_1}}_F  \nonumber \\
		=&\norm{\mathbf{Y}_{\mathcal{C}_1}-(\mathbf{Y}_{\mathcal{C}_2} \mathbf{Y}_{\mathcal{C}_2}^T+\lambda \mathbf{I}-\lambda \mathbf{I})(\mathbf{Y}_{\mathcal{C}_2}\mathbf{Y}_{\mathcal{C}_2}^T+\lambda \mathbf{I})^{-1}\mathbf{Y}_{\mathcal{C}_1}}_F\nonumber \\
		=&\lambda \norm{(\mathbf{Y}_{\mathcal{C}_2}\mathbf{Y}_{\mathcal{C}_2}^T+\lambda \mathbf{I})^{-1}\mathbf{Y}_{\mathcal{C}_1}}_F\nonumber  \\
		<&\lambda\norm{[(\mathbf{Y}_{\mathcal{C}_2}\mathbf{Y}_{\mathcal{C}_2}^T+\lambda \mathbf{I})^{-1}-(\mathbf{U}_1\hat{\mathbf{B}}_2 \hat{\mathbf{B}}_2^T\mathbf{U}_1^T+\lambda \mathbf{I})^{-1}]||_F||\mathbf{Y}_{\mathcal{C}_1}}_F\nonumber  \\
		&+\lambda \norm{(\mathbf{U}_1\hat{\mathbf{B}}_2 \hat{\mathbf{B}}_2^T\mathbf{U}_1^T+\lambda \mathbf{I})^{-1}\mathbf{Y}_{\mathcal{C}_1}}_F \nonumber \\
		<& \lambda\norm{(\mathbf{Y}_{\mathcal{C}_2}\mathbf{Y}_{\mathcal{C}_2}^T+\lambda \mathbf{I})^{-1}-(\mathbf{U}_1\hat{\mathbf{B}}_2 \hat{\mathbf{B}}_2^T\mathbf{U}_1^T+\lambda \mathbf{I})^{-1}}_F \sqrt{d_{\max}+1} \nonumber \\
		& + \lambda \norm{(\mathbf{U}_1\hat{\mathbf{B}}_2 \hat{\mathbf{B}}_2^T\mathbf{U}_1^T+\lambda \mathbf{I})^{-1}\mathbf{U}_1 \mathbf{\hat{B}}_1 }_F \nonumber \\
		& + \lambda \norm{(\mathbf{U}_1\hat{\mathbf{B}}_2 \hat{\mathbf{B}}_2^T\mathbf{U}_1^T+\lambda \mathbf{I})^{-1}}_F \norm{\mathbf{\hat{E}}_1}_F
		\label{eq:thm-2-pt3}.
		\end{align}
		
		Now we are going to complete our proof in $3$ steps. Unless specified otherwise, the following Step $1$ to Step $3$ are derived conditioning on $\mathcal{E}_2$, $\mathcal{E}_3$ and $\mathcal{E}_4$.
		
		\noindent
		\textbf{Step 1}: We can rewrite the first term in \eqref{eq:thm-2-pt3} as the following  term
		\[
		\lambda \norm{(\mathbf{G}_2+\mathbf{H})^{-1}-\mathbf{H}^{-1}}_F  \sqrt{d_{\max} + 1},
		\]
		where $\mathbf{H}=\mathbf{U}_1\hat{\mathbf{B}}_2 \hat{\mathbf{B}}_2^T\mathbf{U}_1^T+\lambda \mathbf{I}$, and $\mathbf{G}_2=\mathbf{Y}_{\mathcal{C}_2}\mathbf{Y}_{\mathcal{C}_2}^T-\mathbf{U}_1\hat{\mathbf{B}}_2 \hat{\mathbf{B}}_2^T\mathbf{U}_1^T=\hat{\mathbf{E}}_2 \hat{\mathbf{B}}_2^T \mathbf{U}_1^T+\mathbf{U}_1\hat{\mathbf{B}}_2\hat{\mathbf{E}}_2^T+\hat{\mathbf{E}}_2\hat{\mathbf{E}}_2^T$. Note that the normalizing constant of each column of $\lbrace \hat{\mathbf{E}}_{i} \rbrace_{i=1}^n$ are bounded in $[1-g,1+g]$. We then have the following relations
		\begin{align}\label{sub:eq:thm2-g2-bound}
		\norm{\mathbf{G}_2}_F &\leq \norm{\hat{\mathbf{E}}_2}_F \norm{\hat{\mathbf{B}}_2^T \mathbf{U}_1^T}_F+\norm{\mathbf{U}_1\hat{\mathbf{B}}_2+\hat{\mathbf{E}}_2}_F \norm{\hat{\mathbf{E}}_2^T}_F \nonumber \\
		& \leq \frac{(2g - g^2)(d_{max}+1)}{(1-g)^2}.
		\end{align}
		The above analysis used triangle inequalities and the bounds of normalizing constants. 
		
		Using the fact that
		\begin{equation*}
		\| \mathbf{H}^{-1} \|_F < \sqrt{\frac{d(1+g)^4}{q_0^2} + \frac{D-d}{\lambda^2}}
		\end{equation*}
		and inequality \eqref{sub:eq:thm2-g2-bound}, we have the following inequalities
		\begin{align*}
		\norm{\mathbf{H}^{-1} \mathbf{G}_2}_F \leq \norm{\mathbf{H}^{-1}}_F \norm{\mathbf{G}_2}_F  = \frac{(2g-g^2)\left(d_{max}+1\right)}{2(1-g)} \cdot \sqrt{\frac{d(1+g)^4}{q_0^2} + \frac{D-d}{\lambda^2}}=:f(d) < \frac{1}{2}.
		\end{align*}
		Therefore $\lim_{m \rightarrow \infty} (\mathbf{H}^{-1}\mathbf{G}_2)^m=\mathbf{0}$. From Theorem $4.29$ in \cite{schott2016matrix} we know $(\mathbf{I}+\mathbf{H}^{-1}\mathbf{G}_2)^{-1}=\sum_{j=0}^{\infty}(\mathbf{H}^{-1}\mathbf{G}_2)^j$ and
		\begin{align*}
		\norm{(\mathbf{G}_2+\mathbf{H})^{-1}-\mathbf{H}^{-1}}_F&=\norm{\mathbf{H}^{-1}\mathbf{G}(\mathbf{I}+\mathbf{H}^{-1}\mathbf{G}_2)^{-1}\mathbf{H}^{-1}}_F \\
		& \leq \norm{\sum_{j=1}^{\infty}(\mathbf{H}^{-1}\mathbf{G}_2)^j}_F \norm{\mathbf{H}^{-1}}_F\\
		& <\frac{f(d)}{1-f(d)}\sqrt{\frac{d(1+g)^4}{q_0^2} + \frac{D-d}{\lambda^2}}.
		\end{align*}
		We then have for the first term in $\eqref{eq:thm-2-pt3}$ 
		\begin{align*}
		\lambda \norm{(\mathbf{G}_2+\mathbf{H})^{-1}-\mathbf{H}^{-1}}_F \sqrt{d_{\max}+1} & <  \frac{f(d)\sqrt{d_{max}+1}}{1-f(d)} \cdot \sqrt{\frac{d(1+g)^4\lambda^2}{q_0^2} + D-d}.
		\end{align*}
		For the second term in \eqref{eq:thm-2-pt3} we have
		\begin{align*}
		\lambda \norm{ \left( \mathbf{U}_1\hat{\mathbf{B}}_2 \hat{\mathbf{B}}_2^T\mathbf{U}_1^T+\lambda \mathbf{I} \right)^{-1} \hat{\mathbf{B}}_1}_F  &=  \norm{\hat{\mathbf{B}}_1-\hat{\mathbf{B}}_2 \hat{\mathbf{B}}_2^T(\hat{\mathbf{B}}_2 \hat{\mathbf{B}}_2^T +\lambda \mathbf{I})^{-1}\hat{\mathbf{B}}_1}_F \\
		&\leq \norm{\mathbf{I}-\hat{\mathbf{B}}_2 \hat{\mathbf{B}}_2^T(\hat{\mathbf{B}}_2 \hat{\mathbf{B}}_2^T +\lambda \mathbf{I})^{-1}}_F \norm{\hat{\mathbf{B}}_1}_F  \leq \frac{\lambda (1+g)^2 \sqrt{d(d_{max}+1)}}{q_0(1-g)}.
		\end{align*}
		For the third term in \eqref{eq:thm-2-pt3} we have
		\begin{align*}
		\lambda \norm{(\mathbf{U}_1\hat{\mathbf{B}}_2 \hat{\mathbf{B}}_2^T\mathbf{U}_1^T+\lambda \mathbf{I})^{-1}}_F \norm{\mathbf{\hat{E}}_1}_F \leq g\sqrt{\frac{d(1+g)^4\lambda^2}{q_0^2} + D-d}.
		\end{align*}
		Hence by our assumption, equation \eqref{eq:thm-2-pt3} can be upper bounded by the following term
		\[
		\frac{3\lambda (1+g)^2 \sqrt{d(d_{max}+1)}}{q_0(1-g)},
		\]
		which is deterministic and does not depend on the choices of $\lbrace \mathbf{B}_i \rbrace_{i=1}^n$ and $\lbrace \mathbf{U}_k \rbrace_{k=1}^K$. The distance function in \eqref{eq:thm2-eq1} has two parts which are symmetric, therefore we set 
		\[
		l :=\frac{6\lambda (1+g)^2 \sqrt{d(d_{max}+1)}}{q_0(1-g)}> d(\mathbf{Y}_{\mathcal{C}_i},\mathbf{Y}_{\mathcal{C}_j})_{(i,j) \in \mathcal{J}}.
		\]
		
		\noindent 
		\textbf{Step 2:} Now we consider $\mathbb{P}\left[ d \left(\mathbf{Y}_{\mathcal{C}_1},\mathbf{Y}_{\mathcal{C}_3}\right) \leq l \middle | \mathcal{E}_2, \mathcal{E}_3, \mathcal{E}_4 \right]$. We explicitly write $d(\mathbf{Y}_{\mathcal{C}_1},\mathbf{Y}_{\mathcal{C}_3})$ as 
		\begin{align}
		\norm{\mathbf{Y}_{\mathcal{C}_1}-\mathbf{Y}_{\mathcal{C}_3}(\mathbf{Y}_{\mathcal{C}_3}^T \mathbf{Y}_{\mathcal{C}_3}+\lambda \mathbf{I})^{-1}\mathbf{Y}_{\mathcal{C}_3}^T\mathbf{Y}_{\mathcal{C}_1}}_F+ \norm{\mathbf{Y}_{\mathcal{C}_3}-\mathbf{Y}_{\mathcal{C}_1}(\mathbf{Y}_{\mathcal{C}_1}^T \mathbf{Y}_{\mathcal{C}_1}+\lambda \mathbf{I})^{-1}\mathbf{Y}_{\mathcal{C}_1}^T\mathbf{Y}_{\mathcal{C}_3}}_F. \label{eq:thm2-eq2}
		\end{align}
		Note the following relation 
		\begin{align} \label{eq:thm2-prob-0}
		\mathbb{P} \left[d(\mathbf{Y}_{\mathcal{C}_1},\mathbf{Y}_{\mathcal{C}_3})  \leq l | \mathcal{E}_2, \mathcal{E}_3 \right] \leq & \mathbb{P} \left[\norm{\mathbf{Y}_{\mathcal{C}_1}-\mathbf{Y}_{\mathcal{C}_3}(\mathbf{Y}_{\mathcal{C}_3}^T \mathbf{Y}_{\mathcal{C}_3}+\lambda \mathbf{I})^{-1}\mathbf{Y}_{\mathcal{C}_3}^T\mathbf{Y}_{\mathcal{C}_1}}_F \leq \frac{l}{2} \middle | \mathcal{E}_2, \mathcal{E}_3,\mathcal{E}_4 \right] \nonumber  \\
		& + \mathbb{P}\left[\norm{\mathbf{Y}_{\mathcal{C}_3}-\mathbf{Y}_{\mathcal{C}_1}(\mathbf{Y}_{\mathcal{C}_1}^T \mathbf{Y}_{\mathcal{C}_1}+\lambda \mathbf{I})^{-1}\mathbf{Y}_{\mathcal{C}_1}^T\mathbf{Y}_{\mathcal{C}_3}}_F \leq \frac{l}{2} \middle | \mathcal{E}_2, \mathcal{E}_3,\mathcal{E}_4 \right]. 
		\end{align}
		To bound the first term in \eqref{eq:thm2-eq2}, the following facts come from the triangle inequality
		\begin{align*}
		&\norm{\mathbf{Y}_{\mathcal{C}_1}-\mathbf{Y}_{\mathcal{C}_3} \left(\mathbf{Y}_{\mathcal{C}_3}^T\mathbf{Y}_{\mathcal{C}_3}+\lambda \mathbf{I} \right)^{-1}\mathbf{Y}_{\mathcal{C}_3}^T \mathbf{Y}_{\mathcal{C}_1}}_F \nonumber \\
		=&\lambda \norm{\left( \mathbf{Y}_{\mathcal{C}_3}\mathbf{Y}_{\mathcal{C}_3}^T+\lambda \mathbf{I} \right)^{-1}\mathbf{Y}_{\mathcal{C}_1}}_F\nonumber  \\
		>& \lambda \norm{\left(\mathbf{U}_2\hat{\mathbf{B}}_3 \hat{\mathbf{B}}_3^T\mathbf{U}_2^T+\lambda \mathbf{I}\right)^{-1} \mathbf{U}_1\hat{\mathbf{B}}_1}_F -\lambda \norm{\left( \mathbf{U}_2\hat{\mathbf{B}}_3 \hat{\mathbf{B}}_3^T\mathbf{U}_2^T+\lambda \mathbf{I} \right)^{-1}}_F \norm{\hat{\mathbf{E}}_1}_F \nonumber \\
		-&\lambda\norm{\left[\left( \mathbf{Y}_{\mathcal{C}_3}\mathbf{Y}_{\mathcal{C}_3}^T+\lambda \mathbf{I} \right) ^{-1}- \left( \mathbf{U}_2\hat{\mathbf{B}}_3 \hat{\mathbf{B}}_3^T\mathbf{U}_2^T+\lambda \mathbf{I} \right)^{-1}\right]}_F \sqrt{d_{\max}+1}.
		\end{align*}
		The last two terms of the line above are upper bounded by $\frac{\lambda (1+g)^2 \sqrt{d(d_{max}+1)}}{q_0(1-g)}$ as before, and the first term can be bounded by the following relations
		\begin{align}
		& \lambda \norm{\left( \mathbf{U}_2\hat{\mathbf{B}}_3 \hat{\mathbf{B}}_3^T\mathbf{U}_2^T+\lambda \mathbf{I} \right)^{-1} \mathbf{U}_1\hat{\mathbf{B}}_1}_F \nonumber \\
		\geq & \norm{\mathbf{U}_1 \hat{\mathbf{B}}_1-\mathbf{U}_2 \mathbf{U}_2^T \mathbf{U}_1 \hat{\mathbf{B}}_1}_F - \lambda \norm{\mathbf{U}_2 \left( \hat{\mathbf{B}}_3 \hat{\mathbf{B}}_3^T + \lambda \mathbf{I}  \right)^{-1} \mathbf{U}_2^T \mathbf{U}_1 \hat{\mathbf{B}}_1}_F \nonumber \\
		> & \norm{\mathbf{U}_1 \hat{\mathbf{B}}_1-\mathbf{U}_2 \mathbf{U}_2^T \mathbf{U}_1 \hat{\mathbf{B}}_1}_F- \frac{\lambda (1+g)^2 \sqrt{d(d_{max}+1)}}{q_0(1-g)} \label{eq:thm2-eq3},
		\end{align}
		where inequality \eqref{eq:thm2-eq3} comes from the following relations 
		\begin{align*}
		\lambda \norm{\mathbf{U}_2 \left( \hat{\mathbf{B}}_3 \hat{\mathbf{B}}_3^T + \lambda \mathbf{I}  \right)^{-1} \mathbf{U}_2^T \mathbf{U}_1 \hat{\mathbf{B}}_1}_F  & \leq \lambda \norm{\left( \hat{\mathbf{B}}_3 \hat{\mathbf{B}}_3^T + \lambda \mathbf{I}  \right)^{-1}}_F  \norm{\mathbf{U}_1\hat{\mathbf{B}}_1}_F\\ & \leq\lambda \frac{\sqrt{d}(1+g)^2}{q_0}\frac{\sqrt{d_{max}+1}}{1-g} = \frac{\lambda (1+g)^2 \sqrt{d(d_{max}+1)}}{q_0(1-g)}.
		\end{align*} 
		
		For the first term in  \eqref{eq:thm2-eq3} we have
		\begin{align*}
		\norm{\mathbf{U}_1 \hat{\mathbf{B}}_1-\mathbf{U}_2 \mathbf{U}_2^T \mathbf{U}_1 \hat{\mathbf{B}}_1}_F &= \sqrt{Tr\left[\hat{\mathbf{B}}_1^T\hat{\mathbf{B}}_1-\hat{\mathbf{B}}_1^T\mathbf{U}_1^T\mathbf{U}_2 \mathbf{U}_2^T \mathbf{U}_1 \hat{\mathbf{B}}_1 \right]} \\
		& = \norm{\sqrt{\mathbf{I}-\mathbf{\Lambda}_{12}^2}\tilde{\mathbf{B}}_1 \mathbf{W}}_F \geq \frac{\norm{\sqrt{\mathbf{I}-\mathbf{\Lambda}_{12}^2}\tilde{\mathbf{B}}_1}_F }{1+g},
		\end{align*}
		where $\mathbf{W}$ is the diagonal matrix such that $W_{jj} = \frac{1} {\norm{[\hat{\mathbf{B}}_{1}]_{j}}_2}$ ($[\hat{\mathbf{B}}_{1}]_j$ is the $j$-th column of $\hat{\mathbf{B}}_{1}$, $j=1,..,d_{max}+1$), $\tilde{\mathbf{B}}_1 =\mathbf{V} \mathbf{B}_1$ is a orthogonal transformation of $\mathbf{B}_1$ (here $\mathbf{V}$ is the right orthogonal matrix in the svd of $\mathbf{U}_2^T \mathbf{U}_1$), and $\mathbf{\Lambda}_{12}$ is the diagonal matrix such that $\left[ \mathbf{\Lambda}_{12} \right]_{ii} = \lambda_i^{(12)}$, $i=1,...,d$. Therefore, eventually the first term at the RHS of \eqref{eq:thm2-prob-0} can be upper bounded by 
		\begin{align*}
		&\mathbb{P} \left[\norm{\sqrt{\mathbf{I}-\mathbf{\Lambda}_{12}^2}\tilde{\mathbf{B}}_1}_F \leq  \frac{6\lambda (1+g)^2 \sqrt{d(d_{max}+1)}}{q_0(1-g)} \right]. 
		\end{align*}
		Using Assumption A4, Lemma~\ref{lem:1} and arguments similar to the proof of Theorem~\ref{sub:thm:sub-preserve-noisy}, we know the quantity above is upper bounded by 
		\begin{align*}
		&\mathbb{P} \left[\norm{\sqrt{\mathbf{I}-\mathbf{\Lambda}_{12}^2}\mathbf{v}}_F \leq \sqrt{1 - T_l^2} \right] \leq 2e^{-{\epsilon_1}^2}, 
		\end{align*}
		where $\epsilon_1$ is defined in \eqref{sub:eq:thm1-epsilon} with $g_1$ replaced by $T_l$, and $\mathbf{v}$ is the first column of $\tilde{\mathbf{B}}_1$. Using analogous manipulations we obtain similar results for the second term in \eqref{eq:thm2-prob-0}. Therefore
		$\mathbb{P}\left[d(\mathbf{Y}_{\mathcal{C}_1},\mathbf{Y}_{\mathcal{C}_3})  \leq l | \mathcal{E}_2, \mathcal{E}_3, \mathcal{E}_4 \right]  \leq 4e^{-{\epsilon_1}^2}$. 
		
		\noindent
		\textbf{Step 3:} Now we are going to lower bound $\mathbb{P}\left[\mathcal{E}_2, \mathcal{E}_3, \mathcal{E}_4 \right]$ from the fact $\mathbb{P}\left[\mathcal{E}_2, \mathcal{E}_3, \mathcal{E}_4 \right] \geq 1 - \mathbb{P}[\mathcal{E}^{\complement}_2] - \mathbb{P}[\mathcal{E}^{\complement}_3] - \mathbb{P}[\mathcal{E}^{\complement}_4]$. 
		
		Just as in the proof of Theorem~\ref{sub:thm:sub-preserve-noisy} we have $
		\mathbb{P} \left[\sigma \mathbf{e}_{i}^{(k)} \geq g \right] \leq 2e^{-t^2}$,
		where
		\[
		t = \frac{D \left( \frac{g^2}{D\sigma^2}-1 \right)}{2 \left( \sqrt{D}+\frac{Dg^2}{\sigma^2\sqrt{d}}+\sqrt{d} \right)}.
		\]
		From Assumption $A4$ we know $2e^{-t^2} \leq \frac{2}{N^{\left(1 + \frac{\eta}{2+\eta}\right)^2}}$. Using union bound inequality we have
		\begin{align}\label{sub:eq:thm2-e2}
		\mathbb{P} \left[\mathcal{E}^{\complement}_2 \right] \leq \frac{2N}{N^{\left(1 + \frac{\eta}{2+\eta}\right)^2}}.
		\end{align}
		From Lemma~\ref{sub:lem:informative} we have 
		\begin{align}\label{sub:eq:thm2-e3}
		\mathbb{P} \left[\mathcal{E}^{\complement}_3 \right] \leq \frac{2n}{N^{t^2/2}} + \sum_{k=1}^K n_k \left( \frac{N_k-d_{max}}{d_{max}(N_k+1)(N_k^{\rho}-1)^2} + 2(N_k-1) e^{-{\epsilon_2^2}} \right) .
		\end{align}
		From our assumption we have 
		\begin{align}\label{sub:eq:thm2-e4}
		\mathbb{P} \left[\mathcal{E}^{\complement}_4 \right] \leq p_s.
		\end{align}
		Combing \eqref{sub:eq:thm2-e2}, \eqref{sub:eq:thm2-e3} and \eqref{sub:eq:thm2-e4} we know
		\begin{align*}
		\mathbb{P}[\mathcal{E}_1] \geq &  1 - \sum_{\forall (i,j) \in \mathcal{I}} \mathbb{P}\left[d(\mathbf{Y}_{\mathcal{C}_i},\mathbf{Y}_{\mathcal{C}_j}) \leq l \middle| \mathcal{E}_2,\mathcal{E}_3, \mathcal{E}_4 \right] - \mathbb{P} \left[{\mathcal{E}}^\complement_2 \right] - \mathbb{P} \left[{\mathcal{E}}^\complement_3 \right] - \mathbb{P} \left[{\mathcal{E}}^\complement_4\right]  \nonumber \\
		\geq &  1 - 4n(n-1)e^{-\epsilon_1^2} -\frac{2n}{N^{t^2/2}} - \sum_{k=1}^K n_k \left( \frac{N_k-d_{max}}{d_{max}(N_k+1)(N_k^{\rho}-1)^2} +2(N_k-1) e^{-{\epsilon_2^2}} \right) \nonumber \\
		& -\frac{2N}{N^{\left(1 + \frac{\eta}{2+\eta}\right)^2}} -p_s.
		\end{align*}
		This completes the proof.
	\end{proof}
	
	\newpage
	\section{Residual Minimization by Ridge Regression}\label{sub:sec:RMRR}

	In this section we provide the algorithm for classifying the out-of-sample points.
	
	\IncMargin{1em}
	\begin{algorithm}
		\caption{Residual Minimization by Ridge Regression (RMRR) algorithm.}\label{sub:algo:REM}
		\SetKwData{Left}{left}\SetKwData{This}{this}\SetKwData{Up}{up}
		\SetKwFunction{Union}{Union}\SetKwFunction{FindCompress}{FindCompress}
		\SetKwInOut{Input}{input}\SetKwInOut{Output}{output}
		\Input{$\mathbf{Y}$ to be classified, $\mathbf{R}$ and $\boldsymbol{\ell}$  are the training data and labels, $m$ and $\lambda$ are the residual minimization and regularization parameters}
		\Output{The label vector $\boldsymbol{\ell}$ of all points in $\mathbf{Y}$}
		
		1. Generate subsets of training data\\
		\For  {$k=1$ \KwTo $K$}   {
			Uniformly sample $m$ points from the $k$-th cluster in the training set $\mathbf{R}$, denote this sampled set as $\mathbf{R}_{k}$;
		}
		
		2. Compute the projection matrix for each cluster \\
		\For  {$k=1$ \KwTo $K$}   {
			$\mathbf{P}_k := \mathbf{R}_k (\mathbf{R}_k^T \mathbf{R}_k + \lambda \mathbf{I})^{-1} \mathbf{R}_k^T$
		}
		
		3.  Compute residuals for points in $Y$, here $N$ is the number of points in $\mathbf{Y}$  \\
		\For  {$i=1$ \KwTo $N$}   {
			\For  {$k=1$ \KwTo $K$}{
				$\mathbf{r}_i(k):=(\mathbf{I}-\mathbf{P}_k)\mathbf{y}_i$;
			}
		}
		
		4. Assign labels through minimum residual\\
		\For  {$i=1$ \KwTo $N$}   {
			$\ell_i = \arg \min_k  \mathbf{r}_i(k) $;
		}
	\end{algorithm}

	\section{Additional Numerical Results}\label{sub:sec:more-numerical}
	In this section, we present  additional numerical results. Results for some algorithms are omitted for certain datasets due to the limitations on computational resources. Specifically, the additional results are presented in Table~\ref{tab:additional-yale}, Table~\ref{tab:additional-zipcode} and Table~\ref{tab:additional-mnist}.
	
	\begin{table}
		\centering
		\begin{tabular}{|c|c|c|c|c|}
			\hline
			\textbf{Method}    & \textbf{Accuracy (\%)}         & \textbf{Accuracy-Sub (\%)}                                     & \textbf{NMI (\%)}                                                  & \textbf{Runtime (sec.)} \\ \hline
			SBSC-SSC  & \begin{tabular}[c]{@{}c@{}}20.99  \\ (1.02)\end{tabular} & \begin{tabular}[c]{@{}c@{}} 22.5\\ (1.29)\end{tabular}   & \begin{tabular}[c]{@{}c@{}} 34.24 
				\\ (1.15)\end{tabular} & 56    \\ \hline
			
			SSSC  & \begin{tabular}[c]{@{}c@{}}49.33\\ (2.51)\end{tabular} & \begin{tabular}[c]{@{}c@{}} 56.54 \\ (1,77)\end{tabular}   & \begin{tabular}[c]{@{}c@{}} 52.82\\ (2.21)\end{tabular} & 22 \\ \hline
			LRR  & 55.63 &  NA  & 64.02 & 29 \\ \hline
			
			LSR  & 54.11 &  NA  & 65.12 & 8 \\ \hline
		\end{tabular}
		\caption{\textbf{Additional Results on Extended Yale B}}
		\label{tab:additional-yale}
	\end{table}
	
	\begin{table}
		\centering
		\begin{tabular}{|c|c|c|c|c|}
			\hline
			\textbf{Method}    & \textbf{Accuracy (\%)}         & \textbf{Accuracy-Sub (\%)}                                     & \textbf{NMI (\%)}                                                  & \textbf{Runtime (sec.)} \\ \hline
			SBSC(6)  & \begin{tabular}[c]{@{}c@{}}75.16  \\ (3.28)\end{tabular} & \begin{tabular}[c]{@{}c@{}} 72.62\\ (1.52)\end{tabular}   & \begin{tabular}[c]{@{}c@{}} 78.79 \\ (1.67)\end{tabular} & 63  \\ \hline
			
			SBSC-DSC(6)  & \begin{tabular}[c]{@{}c@{}}64.25 \\ (1.25)\end{tabular} & \begin{tabular}[c]{@{}c@{}} 65.34\\ (1.86)\end{tabular}   & \begin{tabular}[c]{@{}c@{}} 72.34 \\ (0.76)\end{tabular} & 388  \\ \hline
			
			SBSC-SSC(1)  & \begin{tabular}[c]{@{}c@{}}55.24 \\ (1.34)\end{tabular} & \begin{tabular}[c]{@{}c@{}} 61.44\\ (2.36)\end{tabular}   & \begin{tabular}[c]{@{}c@{}} 45.18 \\ (1.42)\end{tabular} & 117  \\ \hline
			
			SBSC-SSC(6)  & \begin{tabular}[c]{@{}c@{}}71.07  \\ (0.94)\end{tabular} & \begin{tabular}[c]{@{}c@{}} 68.65 \\ (1.21)\end{tabular}   & \begin{tabular}[c]{@{}c@{}} 67.39 \\ (1.19)\end{tabular} & 703  \\ \hline
			
			STSC(6)  & \begin{tabular}[c]{@{}c@{}}57.76
				\\ (1.15)\end{tabular} & \begin{tabular}[c]{@{}c@{}} 60.1\\ (1.8)\end{tabular}   & \begin{tabular}[c]{@{}c@{}} 60.4 \\ (1.59)\end{tabular} & 13\\ \hline
			
			SDSC(6)  & \begin{tabular}[c]{@{}c@{}}52 
				\\ (2.63)\end{tabular} & \begin{tabular}[c]{@{}c@{}} 51.59\\ (1.28)\end{tabular}   & \begin{tabular}[c]{@{}c@{}} 63.28\\ (1.26)\end{tabular} & 51
			\\ \hline
			
			SSSC(1)  & \begin{tabular}[c]{@{}c@{}}41.52 \\ (5.92)\end{tabular} & \begin{tabular}[c]{@{}c@{}} 44.86\\ (7.06)\end{tabular}   & \begin{tabular}[c]{@{}c@{}} 38.22\\ (3.7)\end{tabular} & 25\\ \hline
			
			SSSC(6)  & \begin{tabular}[c]{@{}c@{}}44.43
				\\ (4)\end{tabular} & \begin{tabular}[c]{@{}c@{}} 44.06\\ (2.53)\end{tabular}   & \begin{tabular}[c]{@{}c@{}} 42.61 \\ (2.23)\end{tabular} & 150\\ \hline  
			
			SLRR(6)  & \begin{tabular}[c]{@{}c@{}}63.7 \\ (3.74)\end{tabular} & \begin{tabular}[c]{@{}c@{}} 63.85\\ (1.74)\end{tabular}   & \begin{tabular}[c]{@{}c@{}} 69.25\\ (1.86)\end{tabular} & 46  \\ \hline
			
			SLSR(6)  & \begin{tabular}[c]{@{}c@{}}60.71	\\ (1.04)\end{tabular} & \begin{tabular}[c]{@{}c@{}} 59.43 \\ (0.8)\end{tabular}   & \begin{tabular}[c]{@{}c@{}} 66.39\\ (1.08)\end{tabular} & 26 \\ \hline
			LRR  & 53.25 &  NA  & 53.53 & 401 \\ \hline
			
			LSR  & 58.91 &  NA  & 61.56 & 192 \\ \hline
		\end{tabular}
		\caption{\textbf{Additional Results on Zipcode}}
		\label{tab:additional-zipcode}
	\end{table}
	
	\begin{table}
		\centering
		\begin{tabular}{|c|c|c|c|c|}
			\hline
			\textbf{Method}    & \textbf{Accuracy (\%)}         & \textbf{Accuracy-Sub (\%)}                                     & \textbf{NMI (\%)}                                                  & \textbf{Runtime (sec.)} \\ \hline
			SBSC-SSC & \begin{tabular}[c]{@{}c@{}} 84.95\\ (4.51) \end{tabular} & \begin{tabular}[c]{@{}c@{}} 86.48\\(4.2)\end{tabular}   & \begin{tabular}[c]{@{}c@{}} 73.71 \\ (2.06)\end{tabular} & 834
			\\ \hline
			
			SSSC(1)  & \begin{tabular}[c]{@{}c@{}}33.26\\ (2.15)\end{tabular} & \begin{tabular}[c]{@{}c@{}} 77.22\\ (3.9)\end{tabular}   & \begin{tabular}[c]{@{}c@{}} 13.59\\ (1.41)\end{tabular} & 43 \\ \hline
			
			SSSC(6)  & \begin{tabular}[c]{@{}c@{}}48.49 \\ (2.75)\end{tabular} & \begin{tabular}[c]{@{}c@{}} 79.06\\ (1.63)\end{tabular}   & \begin{tabular}[c]{@{}c@{}} 30.41 \\ (2.04)\end{tabular} & 259\\ \hline
		\end{tabular}
		\caption{\textbf{Additional Results on MNIST}}
		\label{tab:additional-mnist}
	\end{table}
	
	\section{Additional Technical Discussion}\label{sub:sec:more-technical}
	\subsection{The $\epsilon_1$ in Theorem~\ref{sub:thm:sub-preserve-noisy}}
	In this section, we will show that under mild conditions, $\epsilon_1$ in \eqref{sub:eq:thm1-epsilon} grows at least linear in $\sqrt{d}$. For ease of notation, we write $r_{i} = \left(g_{1}^2-\lambda_i^{(1)2}\right)_+$ and $s_i =   \left(g_{1}^2-\lambda_i^{(1)2}\right)_-$, $i=1,..,d$. WLOG assume $\epsilon_1$ is evaluated at $k=1$.
	
	\noindent
	\textbf{Main result}: If there exist constants $c_1 \in (0,g_1]$, $c_2 \in (0, 1)$ and $c_3 > 0$ such that $\sum_{i=1}^d r_i \geq c_1 d$, $\frac{\sum_{i=1}^d s_i}{\sum_{i=1}^d r_i} \leq c_2$ and $c_3 d > g_1$, then we have 
	\[
	\epsilon_1 \geq \frac{(1-c_2)\sqrt{d}}{2\sqrt{\frac{(c_1 + c_3)g_1}{c_1^2}} + \sqrt{\frac{4(c_1 + c_3)g_1}{c_1^2}+ \frac{2}{c_1}}}.
	\]
	\begin{proof}
		Note that 
		\begin{align}\label{sub:eq:appendixD-1}
		\epsilon_1 = \frac{1 - \frac{\sum_{i=1}^d s_i}{\sum_{i=1}^d r_i}}{2\sqrt{\frac{\sum_{i=1}^d r_i^2}{(\sum_{i=1}^d r_i)^2}} + \sqrt{\frac{4\sum_{i=1}^d r_i^2}{(\sum_{i=1}^d r_i)^2} + \frac{2}{\sum_{i=1}^d r_i}}}.
		\end{align}
		Define $f: \mathcal{V} \rightarrow R$, where $f(\mathbf{v}) = \frac{\sum_{i=1}^d v_i^2}{(\sum_{i=1}^d v_i)^2}$, and $\mathcal{V} = \lbrace \mathbf{v} \in [0, g_1]^d: \sum_{i=1}^d v_i = \sum_{i=1}^d r_i \rbrace$. Consider the following $\mathbf{r}^* \in \mathcal{V}$
		\[
		r^*_i =
		\begin{cases*}
		g_1, & if $i \leq \floor{\frac{\sum_{i=1}^d r_i}{g_1}}$,\\
		\sum_{i=1}^d r_i - \floor{\frac{\sum_{i=1}^d r_i}{g_1}} \cdot g_1,        &  $i = \floor{\frac{\sum_{i=1}^d r_i}{g_1}} + 1$, \\
		0, & otherwise.
		\end{cases*}
		\]
		We will prove by contradiction that any maximizer of $f(\cdot)$ is a permutation of $\mathbf{r}^*$.
		
		In fact, assume $\mathbf{r}' \in \mathcal{V}$ also maximizes $f(\cdot)$ but is not a permutation of $\mathbf{r}^*$. Assume there are $m$ terms in $\lbrace r'_i \rbrace_{i=1}^d$ that are equal to $g_1$. Let $r'_1 \leq r'_2$ be the two smallest positive terms of $\lbrace r'_i \rbrace_{i=1}^d$. It is straightforward to see $r'_2 < g_1$. Consequently, we can find a constant $\delta >0$ such that $r'_1-\delta, r'_2+\delta \in (0,g_1)$. Note $\mathbf{r}''=\left[r'_1 - \delta, r'_2 + \delta, r'_3,...,r'_d\right] \in \mathcal{V}$, but $f(\mathbf{r}'') > f(\mathbf{r}')$, which is a contradiction.
		
		Note that $\mathbf{r} \in \mathcal{V}$, we plug $\mathbf{r}^*$ into $f(\cdot)$ and get 
		\begin{align*}
		f(\mathbf{r}) = \frac{\sum_{i=1}^d r_i^2}{(\sum_{i=1}^d r_i)^2} \leq \frac{(\frac{\sum_{i=1}^d r_i}{g_1} + 1)g_1^2}{(\sum_{i=1}^d r_i)^2} \leq \frac{(\frac{c_1 d}{g_1} + 1)g_1^2}{(c_1 d)^2} \leq \frac{(c_1 + c_3)g_1}{c_1^2d}.
		\end{align*}
		
		Finally, from the inequality above and \eqref{sub:eq:appendixD-1} we have
		\begin{align*}
		\epsilon_1 \geq \frac{1 - c_2}{2\sqrt{\frac{(c_1 + c_3)g_1}{c_1^2d}} + \sqrt{\frac{4(c_1 + c_3)g_1}{c_1^2d}+ \frac{2}{c_1 d}}} = \frac{(1-c_2)\sqrt{d}}{2\sqrt{\frac{(c_1 + c_3)g_1}{c_1^2}} + \sqrt{\frac{4(c_1 + c_3)g_1}{c_1^2}+ \frac{2}{c_1}}}.
		\end{align*}
	\end{proof}
	
	\newpage
	\bibliography{ref_subspace}

\end{document}